\documentclass[onecolumn]{IEEEtran}
\usepackage[dvipsnames]{xcolor}
\usepackage{amsmath,amsfonts,mathrsfs,amssymb,bm,amsthm}
\usepackage{hyperref}
\usepackage{amsthm}
\usepackage{algorithm}
\usepackage{array}
\usepackage{subcaption}
\usepackage{textcomp}
\usepackage{stfloats}
\usepackage{url}
\usepackage{verbatim}
\usepackage{graphicx}
\usepackage{cite}
\usepackage{url}
\usepackage{booktabs}
\usepackage{multirow}
\usepackage{dsfont}
\usepackage{optidef}
\usepackage{tikz, pgfplots}
\usepackage{empheq}
\usepackage{centernot}
\usepackage{algpseudocode}
\usepackage{nicematrix}
\usepackage{diagbox}
\usepackage{makecell}
\usepackage{environ}
\usepackage{soul}
\usepackage[normalem]{ulem}

\usetikzlibrary{circuits.ee.IEC}
\usetikzlibrary{matrix}
\usepgfplotslibrary{groupplots}
\usepackage{etoolbox}
\AtBeginEnvironment{bmatrix}{\setlength{\arraycolsep}{1.5pt}}
\usepackage{amsmath,amsfonts,bm}

\def\sgn{{\textnormal{sgn}}}

\def\1{\bm{1}}

\def\ra{{\textnormal{a}}}
\def\rb{{\textnormal{b}}}

\def\re{{\textnormal{e}}}

\def\rh{{\textnormal{h}}}

\def\ru{{\textnormal{u}}}

\def\rw{{\textnormal{w}}}
\def\rx{{\textnormal{x}}}
\def\ry{{\textnormal{y}}}
\def\rz{{\textnormal{z}}}

\def\rvs{{\mathbf{s}}}

\def\rvx{{\mathbf{x}}}

\def\rvz{{\mathbf{z}}}

\def\vdelta{{\bm{\delta}}}

\def\vs{{\bm{s}}}

\def\vv{{\bm{v}}}

\def\vx{{\bm{x}}}

\def\vz{{\bm{z}}}

\def\vDelta{{\bm{\Delta}}}

\def\mF{{\bm{F}}}

\def\mI{{\bm{I}}}

\def\mM{{\bm{M}}}

\DeclareMathAlphabet{\mathsfit}{\encodingdefault}{\sfdefault}{m}{sl}
\SetMathAlphabet{\mathsfit}{bold}{\encodingdefault}{\sfdefault}{bx}{n}

\DeclareMathOperator*{\argmax}{arg\,max}

\usepackage[acronym]{glossaries}

\newacronym{iot}{IoT}{internet of things}
\newacronym{mec}{MEC}{mobile edge computing}
\newacronym{ee}{EE}{early exiting}
\newacronym{soc}{SoC}{state of charge}

\newacronym{api}{API}{application programming interface}
\newacronym{dfs}{DFS}{dynamic frequency scaling}

\newacronym{ann}{ANN}{artificial neural network}
\newacronym{nn}{NN}{neural network}
\newacronym{dnn}{DNN}{deep neural network}

\newacronym{pu}{PU}{processing units}
\newacronym{cpu}{CPU}{central processing unit}
\newacronym{gpu}{GPU}{graphical processing unit}

\newacronym{ai}{AI}{artificial intelligence}
\newacronym{pi}{PI}{policy iteration}
\newacronym{vi}{VI}{value iteration}
\newacronym{rl}{RL}{reinforcement learning}
\newacronym{fl}{FL}{federated learning}
\newacronym{dqn}{DQN}{Deep Q-Network}
\newacronym{med}{MED}{mobile edge device}
\newacronym{sc}{SC}{split computing}
\newacronym{fc}{FC}{fully connected}
\newacronym{conv}{Conv}{convolutional}
\newacronym{fifo}{FIFO}{First-In First-Out}
\newacronym{eh}{EH}{energy-harvesting}
\newacronym{iem}{IEM}{intermittent execution model}
\newacronym{mcu}{MCU}{microcontroller unit}
\newacronym{esd}{ESD}{energy storage device}
\newacronym{es}{ES}{energy storage}
\newacronym{pmf}{pmf}{probability mass function}
\newacronym{pdf}{pdf}{probability density function}
\newacronym{nas}{NAS}{neural architecture search}
\newacronym{flop}{FLOP}{floating point operation}
\newacronym{mbconv}{MBConv}{mobile inverted bottleneck convolution}
\newacronym{ml}{ML}{machine learning}
\newacronym{dl}{DL}{deep learning}
\newacronym{vit}{ViT}{vision transformer}
\newacronym{ehd}{EHD}{energy harvesting device}
\newacronym{rcd}{RCD}{resource-constrained device}
\newacronym{mmux}{MMUX}{model multiplexing}
\newacronym{mms}{MMS}{multi-model selection}
\newacronym{mdp}{MDP}{Markov decision process}
\newacronym{iaw}{IAw}{instance-aware}
\newacronym{iag}{IAg}{instance-agnostic}
\newacronym{inc}{Inc}{incremental}
\newacronym{os}{OS}{one-shot}

\pgfplotsset{
    compat=1.18,
    /pgfplots/legend image code/.code={%
        \draw[mark repeat=2,mark phase=2,#1] 
            plot coordinates {
                (0cm,0cm) 
                (0.3cm,0cm)
            };
    },
}
\usetikzlibrary{shapes,arrows,positioning}
\usetikzlibrary{automata,arrows,positioning,calc}
\usetikzlibrary {arrows.meta}
\usepgfplotslibrary{fillbetween}

\newtheorem{lemma}{Lemma}
\newtheorem{theorem}{Theorem}

\newtheorem{example}{Example}

\newsavebox{\measure@tikzpicture}
\NewEnviron{scaletikzpicturetowidth}[1]{%
  \def\tikz@width{#1}%
  \def\tikzscale{1}\begin{lrbox}{\measure@tikzpicture}%
  \BODY
  \end{lrbox}%
  \pgfmathparse{#1/\wd\measure@tikzpicture}%
  \edef\tikzscale{\pgfmathresult}%
  \BODY
}

\newcommand\bluesout{\bgroup\markoverwith{\textcolor{blue}{\rule[0.5ex]{2pt}{0.4pt}}}\ULon}

\makeatletter
\newcommand*\titleheader[1]{\gdef\@titleheader{#1}}
\AtBeginDocument{%
  \let\st@red@title\@title
  \def\@title{%
    \bgroup\footnotesize\@titleheader\par\egroup
    \vskip.8em\st@red@title}
}
\makeatother

\titleheader{©2024 IEEE. This work has been submitted to the IEEE for possible publication. Copyright may be transferred without notice, after which this version may no longer be accessible.}

\title{Energy-Aware Dynamic Neural Inference}

\author{Marcello Bullo,~\IEEEmembership{Student Member,~IEEE}, Seifallah Jardak, Pietro Carnelli, Deniz G\"und\"uz~\IEEEmembership{Fellow,~IEEE}
\thanks{This work received funding from the EU Horizon 2020 Marie Skłodowska Curie ITN Greenedge (GA. No. 953775), and UKRI through project SONATA (EPSRC-EP/W035960/1). For the purpose of open access, the authors have applied a Creative Commons Attribution (CC BY) license to any Author Accepted Manuscript version arising from this submission.}
\thanks{Marcello Bullo is with the Department of Electrical and Electronic Engineering, Imperial College London, and with Toshiba Bristol Research and Innovation Laboratory, Bristol (email: m.bullo21@imperial.ac.uk, marcello.bullo@toshiba-bril.com)}
\thanks{Seifallah Jardak and Pietro Carnelli are with Toshiba Bristol Research and Innovation Laboratory, Bristol (email: \{name\}.\{surname\}@toshiba-bril.com)}
\thanks{Deniz G\"und\"uz is with the Department of Electrical and Electronic Engineering, Imperial College London (email: d.gunduz@imperial.ac.uk)}
}

\begin{document}

\maketitle

\begin{abstract}

    
    The growing demand for intelligent applications beyond the network edge, coupled with the need for sustainable operation, are driving the seamless integration of \glsdesc*{dl} algorithms into energy-limited, and even energy-harvesting end-devices. However, the stochastic nature of ambient energy sources often results in insufficient harvesting rates, failing to meet the energy requirements for inference and causing significant performance degradation in energy-agnostic systems. To address this problem, we consider an on-device adaptive inference system equipped with an energy-harvester and finite-capacity energy storage. We then allow the device to reduce the run-time execution cost on-demand, by either switching between differently-sized neural networks, referred to as \gls*{mms}, or by enabling earlier predictions at intermediate layers, called \gls*{ee}. 
    The model to be employed, or the exit point is then dynamically chosen based on the energy storage and harvesting process states.
    We also study the efficacy of integrating the prediction confidence into the decision-making process. We derive a principled policy with theoretical guarantees for confidence-aware and -agnostic controllers. Moreover, in multi-exit networks, we study the advantages of taking decisions incrementally, exit-by-exit, by designing a lightweight reinforcement learning-based controller.
    Experimental results show that, as the rate of the ambient energy increases, energy- and  confidence-aware control schemes show approximately $5\%$ improvement in accuracy compared to their energy-aware confidence-agnostic counterparts. 
    Incremental approaches achieve even higher accuracy, particularly when the energy storage capacity is limited relative to the energy consumption of the inference model.
\end{abstract}


\section{Introduction}\label{sec:intro}
\IEEEPARstart{T}{he} widespread presence of interconnected devices, driven by pervasive and ubiquitous computing paradigms, continuously generates an unprecedented volume of data.
\gls*{ml} and \gls*{dl} methodologies have demonstrated substantial efficacy in decoding patterns and extracting knowledge from heterogeneous sensor data \cite{10177897, 9234528}, enabling accurate predictions and informed decisions in many domains such as smart healthcare \cite{9129779, shukla2022review}, human activity recognition \cite{10172911, 8767027}, and intelligent transportation \cite{9509348,9670465}.
 However, energy and computational demands of \gls*{dl} models are typically prohibitive for practical deployment on mobile devices, characterized by limited computing power and constrained memory relative to dedicated servers. For example, in computer vision applications, state-of-the-art vanilla attention-based architectures, such as vision transformers, involve a number of parameters ranging from $86$M to $632$M \cite{dosovitskiy2021an}, with the computational and memory cost of self-attention mechanism increasing quadratically with the image resolution \cite{papa2024survey}. 

\gls*{mec} has emerged to bring intelligence closer to edge devices, \cite{mao2017survey} enhancing their computational capacity by offloading tasks to the network edge. Yet, its reliance on stable connectivity introduces trade-offs between latency and energy efficiency. In harsh channel conditions, offloading demands high transmit power, often exceeding the energy cost of local processing. This limitation underscores the need for an intelligent device-edge continuum \cite{kokkonen2022autonomy}, operating regardless of network quality \cite{sonic}.
In this scenario, \gls*{eh} \cite{ma2019sensing} can be pivotal in achieving energy-autonomy, and ensuring the sustainability and longevity of provided services, particularly those deployed in remote or inaccessible locations whereby regular battery replacement becomes unrealistic and cost-prohibitive \cite{aldin2023comprehensive}.

A major challenge in the deployment of \glspl*{ehd} is the constrained and sporadic nature of the energy they capture.
Therefore, while minimizing the energy consumption is the main goal in \gls*{es}-operated devices in order to extend their lifespan, the primary objective with \gls*{ehd} is the intelligent management of available energy for prolonged operation. In principle, an \gls*{ehd} has access to a potentially infinite energy supply. However, this energy source is intermittent, requiring effective management to ensure stable operation and to mitigate the impacts of energy shortages. This involves developing strategies for dynamic adjustment of device activities based on energy availability.

Recent advancements in intermittent computing \cite{sonic,lucia2017intermittent} have enabled  \gls*{dl} in \gls*{eh} IoT devices \cite{dleh_survey, zhao2022towards,eperceptive, 9401799, 9218526}. Traditional \gls*{dl} models are typically designed for environments with consistent energy sources and do not account for the stochastic availability of energy in \gls*{eh} scenarios. To address this, adaptive \glspl*{dnn} \cite{dynamicnn_survey} can be employed. These models are capable of conditionally reducing the execution cost \textit{on-demand} at inference time, trading-off performance with energy consumption. In general, in adaptive \glspl*{dnn} several computing modes are available for processing sensor data. Associated with each mode is an energy cost, with more costly modes generally yielding more reliable and accurate predictions.

We focus on two different techniques to implement adaptive inference under random energy dynamics: multi-model selection (\gls*{mms}) \cite{lee2019mobisr, mcdnn, 6797059} and early exiting (\gls*{ee}) \cite{teerapittayanon2016branchynet, matsubara_scee_survey, Scardapane2020WhySW}. \gls*{mms} involves switching between different \glspl*{dnn} (computing modes), each with varying energy requirements, depending on the available energy. At the beginning of the inference process, i.e., when a new sample arrives, the system selects one of the available \glspl*{dnn} conditioned on the current energy availability. Multi-exit networks incorporate multiple classifiers (computing modes) at different layers, enabling the inference task to terminate at an earlier stage if a decision is made to reduce the computational workload, thus skipping the subsequent layers. The adaptability of \gls*{ee} makes it particularly appealing for \gls*{eh} systems \cite{bullo, 9401799, 9772720, 9218526}, enabling design and run-time flexibility depending on the \textit{availability of feedback information} in the decision-making process and the \textit{operational granularity} of the action selection. When feedback information is available regarding the prediction confidence at the current exit branch, the system can amortize the instantaneous energy consumption by momentarily pausing or halting the computation pipeline if sufficient reliability in the current prediction is reached. We call this \textit{\gls*{iaw}}  control. If no feedback information is available, then the decisions are based only on average (non-instantaneous) reliability metrics, e.g., accuracy on an unseen dataset available prior to deployment. This is referred to as \textit{\gls*{iag}} control. 
By operational granularity we refer to the adaptivity of the \gls*{ehd} decision making process. When \emph{\gls*{inc}} control is possible, a sequential decision process can be deployed, where the selection of an exit occurs progressively over time. Alternatively, \textit{\gls*{os}} control refers to a non-incremental approach, wherein the decision of a computing mode (exit branch or model) is made immediately upon the arrival of a sample. 
As a consequence, an \gls*{inc} controller operates at a finer level than its \gls*{os} counterpart, with the advantage of observing intermediate information and adjusting its action selection on-the-go. 
\begin{table}[t]
    \centering
    \begin{tabular}{p{2cm}*{3}{c}}
        \toprule
        & \multicolumn{2}{c}{\textbf{Availability of feedback information}} \\
        \cmidrule(lr){2-3}
        \textbf{Operational Granularity} & Instance-Aware (\gls*{iaw}) & Instance Agnostic (\gls*{iag})\\ \cmidrule(lr){1-1}
        \cmidrule(lr){2-3}
        One-shot (\gls*{os})& \gls*{os}-\gls*{iaw}-oracle & \gls*{mms} \\
        Incremental (\gls*{inc}) & \gls*{inc}-\gls*{iaw}-\gls*{ee}-DQN & \gls*{inc}-\gls*{iag}-\gls*{ee} \\ \bottomrule
    \end{tabular}%
    \caption{Summary of control methods studied in this paper.}
    \label{tab:methods_studied}
\end{table}
\begin{figure*}[ht!]
    \centering
    \input{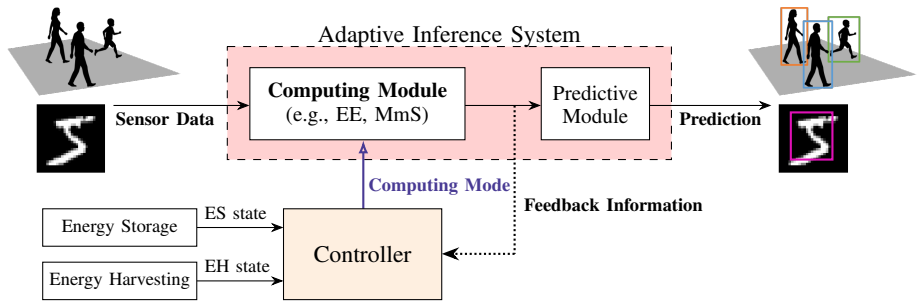}
    \caption{System diagram of an adaptive inference system for resource-constrained devices. Sensor data are processed by a computing module (e.g., \gls*{mms} or \gls*{ee} where the computation mode (e.g., model or exit selection), is regulated by the controller based on the \gls*{es} level, \gls*{eh} dynamics, and potential feedback information (e.g., prediction confidence). In adaptive \glspl*{dnn}, the predictive module represents the sequence of operations converting the output likelihoods into the final prediction. In this scenario, the goal of the controller is to trade-off the prediction accuracy for the processing energy cost.}
    \label{fig:sys_diagram}
\end{figure*}

In this paper, we focus on adapting neural inference workloads of \glspl*{ehd} to the available energy.
%
The system adapts by implementing the specific adaptation strategy (either \gls*{ee} or \gls*{mms}) contingent on the availability of feedback information and the operational granularity. Table~\ref{tab:methods_studied} shows the control schemes resulting from the combination of such characteristics. \gls*{ee}, due to its incremental nature, allows for the integration of feedback information into the decision-making process, thereby enabling instance-aware (\gls*{iaw}) and incremental (\gls*{inc}) control schemes. Conversely, \gls*{mms} cannot function incrementally, as transitioning between models necessitates re-processing the current input instance from the initial input layer of the new model. Consequently, \gls*{mms} operates exclusively in a \gls*{os} manner, which precludes the integration of any feedback information into the decision-making process. Hence, the only \gls*{mms} feasible controller is \gls*{os}-\gls*{iag}-\gls*{mms}, which we simply refer to as \gls*{mms}. Conversely, a \gls*{os}-\gls*{iaw} controller is referred to as an \textit{oracle} due to its impracticality, regardless of whether it pertains to \gls*{ee} or \gls*{mms}. In the absence of feedback information in the decision-making process (\gls*{iag}), \gls*{os}-\gls*{iag}-\gls*{ee} and \gls*{mms} schemes are mathematically equivalent. Therefore, \gls*{mms} is selected to represent \gls*{os}-\gls*{iag} controllers.  
%
%

In the overall system depicted in Figure~\ref{fig:sys_diagram}, we are concerned with the following fundamental question: how should inference in dynamic \glspl*{dnn} be optimized to deliver accurate predictions under stochastic availability of energy? 
Existing works predominantly focus on instance-agnostic short-term decision-making and heuristic approaches like policy networks and gating functions \cite{dynamicnn_survey}. For instance, the model presented in \cite{eperceptive} utilizes a single DNN implementing multi-resolution and \gls*{ee} to optimize energy usage. The exit selection relies on a lookup table of empirically measured charge times for short-term decisions.
Similarly, the authors in \cite{9218526} use an instance-agnostic \gls*{rl}-based policy for exit selection. The approach in \cite{9772720} implements runtime exit selection via an empirical joint threshold on the \gls*{es} level and entropy of prediction likelihoods. These methods are mostly heuristic, and lack theoretical justifications. In contrast, our work addresses the energy-constrained inference problem from a theoretical perspective. A simpler scenario involving a \gls*{dnn} with one early exit is studied our previous work \cite{bullo}.
Here, we delve into the generalized problem of optimal long-term scheduling of a finite number of computing modes, leveraging statistical information on \gls*{eh}. The problem is investigated for all the control schemes in Table~\ref{tab:methods_studied}.
The major contributions of this work are summarized as follows:

\begin{itemize}
    \item We formulate the problem as a sequential decision problem promoting accurate predictions under uncertain energy dynamics by encoding the accuracy of the decisions into the reward signal.
    \item We establish the structure of the optimal policy for the \gls*{mms} system, and show that is monotone in the \gls*{es} level;
    \item To have a principled upper-bound with theoretical guarantees for a multi-exit scheduler, we study a \gls*{os} oracle controller which has full information of the per-instance exit confidences. We extend our previous analysis \cite{bullo} to a more general case with a finite number of exits, characterizing the structure of the optimal policy.
    \item The optimal policy for the theoretical upper bound is based on the true, yet unknown, joint distribution of exit confidences. To numerically compute such a policy, we develop an approximate value iteration (VI) algorithm, which relies on the empirical distribution to estimate the expected value function in the Bellman equation.
    By leveraging the observed data, our method provides an effective approximation to the theoretical model.
    \item By addressing the \gls*{ee} selection as an instance-aware incremental control problem, we derive a sub-optimal policy using a \gls*{dqn} \cite{dqn}, referred to as \gls*{inc}-\gls*{iaw}-\gls*{ee}.
    \item We also consider a \gls*{os}-\gls*{iag} system when the device is instance-agnostic, and empirically assess the benefits of instance-awareness and incremental decisions.
    \item We test and compare these systems on a custom multi-exit EfficientNet-based model \cite{Efficientnet} and TinyImageNet \cite{le2015tiny} dataset, by examining the resulting policies and analyzing the accuracy performance under different \gls*{es} capacities and incoming energy rates.
\end{itemize}

The rest of the paper is organized as follows. Section~\ref{sec:sys_model} provides a detailed description of the system components, and a formal definition of the control schemes in Table~\ref{tab:methods_studied}. In Section~\ref{sec:optimization}, we formulate the optimization problem and establish the main theoretical results on optimal policies for the proposed controllers. Sections \ref{sec:exp_setup} and \ref{sec:exp_results} present the experimental setup and empirical results on the evaluation of our theoretical framework and comparison of the control schemes studied. Section~\ref{sec:conclusion} concludes the article and discuss potential future directions. Detailed derivations, including proofs of theorems, are provided in Appendices.

\section{System Model}\label{sec:sys_model}
In this section, we begin with a description
of the system model, outlining the key components of the proposed adaptive inference system. For clarity, Table~\ref{tab:notation} provides a summary of the main notation used in this work.
\begin{table}[t]
    \centering
    \begin{tabular}{ll}
        \toprule
        \textbf{Notation} & \textbf{Definition} \\
       \midrule
        $t$ &  Time slot index  \\
        $\rw_n$ & $n$-th random input\\
        $t_n$ &  Arrival time index of $\rw_n$  \\
        $K$ &  Number of available computing modes \\
        $\rh_t$ & Energy Harvesting (EH) condition process at $t$ \\
        $\re_t^H$ & Energy harvested at $t$  \\
        $b_{\text{max}}$ & Energy Storage (ES) maximum capacity \\ 
        $p_{\rh}^{h_i,h_j}$ & $\mathbb{P}(\rh_{t+1}=h_j|\rh_t=h_i)$\\
        $p_{\re^H}^{h}(e)$ & $\mathbb{P}(\re^H_t=e|\rh_t=h)$\\
        $\rb_t$ &  \gls*{es} level at $t$ \\
        $\ra_{t_n}$ &  Action selected at $t_n$ by a one-shot (Os) controller\\
        $\alpha_{t}$ &  Incremental sub-action (\textit{pause}/\textit{proceed}) selected at $t$\\
        $\ru_t$ &  Energy cost of $\ra_t$, $\ru_t=u(\ra_t)$ \\
        $\xi_t$ &  Exit index at $t$ (Incremental schemes) \\
        $\tau_t$ & Processing stage at $t$ (Incremental schemes) \\
        $\rvz_t$ & Feedback information at $t$ (prediction confidence)\\
        $\rvz^{(i)}_t$ & $i$-th component of the feedback information at $t$\\
        $\rvx_t$ &  System state at $t$ \\
        $\rvs_t$ & Discrete component of the system state at $t$ \\
        $\hat{\ry}_{t_n}^{(k)}$ & Prediction for $\rw_{n}$ output by the $k$-th comuting mode
    \end{tabular}
    \caption{Summary of used notations and their definitions.}
    \label{tab:notation}
\end{table}

We formulate the problem as a discrete-time \gls*{mdp} with constant-duration time slots indexed by $t\in\mathbb{N}$. We assume that the sensing apparatus (e.g., a camera) monitors the environment at a constant rate,
providing the \gls*{rcd} with the instance $\rw$ (e.g., an image) for processing every $T$ slots. Without any loss of generality, we assume that the data arrival process starts at $t=0$. Let $t_n=nT$, $n\in\mathbb{N}$, be the time index of the $n$-th data arrival. The controller selects the \gls*{dnn} model (\gls*{mms}) or exit branch (\gls*{ee}) at which the computation is halted or temporarily paused, and operates at one of the two granularity levels: (1) singularly and responsively to the arrival of an input sample $\rw$, that is at decision epochs $\{t_n\}_{n\in\mathbb{N}}$, with an inter-action time of $T$ slots (\textit{one-shot}); or (2) incrementally within the interval \([t_n, t_{n+1})\), at intermediate time slots \( t = t_n + \tau_t \), \( \tau_t = 0, \ldots, T-1\), $\forall n\in\mathbb{N}$ (\textit{incremental}). Precisely, $\tau_t=t\bmod T$. The action selection scheme is depicted in Figure~\ref{fig:slots}. 
\begin{figure}[t]
    \centering
    \usetikzlibrary{arrows.meta, positioning}

\begin{center}
\begin{scaletikzpicturetowidth}{0.47\textwidth}
\begin{tikzpicture}[scale=\tikzscale]
    ]
    
    \draw[color=gray] (-2,0) -- (-1.5,0);
    \draw[color=gray] (-.5,0) -- (4.5,0);
    \draw[color=gray] (6.5,0) -- (9.5,0);
    \node[anchor=east] at (-.5,0) {$\ldots$};
    \node[anchor=east,thick, color=NavyBlue] at (-.5,1.5) {$\ldots$};
    \node[anchor=center] at (5.5,0) {$\ldots$};
    \node[anchor=west] at (9.5,0) {$\ldots$};
    \foreach \x/\i [count=\j] in {-2/t_0,
                                  0/t_n,
                                  2/t_n+1,
                                  4/t_n+2,
                                  7/t_n+T-1, 
                                  9/t_{n+1}}  
        \draw (\x, 3pt) -- ++ (0,-6pt) node (b\j) [below] {\scriptsize$\i$};


\foreach \x/\i [count=\j] in {-2//0, 
                                  0/nT,
                                  2/nT+1,
                                  4/nT+2,
                                  7/nT+T-1, 
                                  9/(n+1)T}  
        \draw (\x,5pt) node (b\j) [above,align=center] {\scriptsize$\i$};

    \draw[thick, color=black] (0,5pt) -- ++ (0,-10pt);
    \draw[thick, color=black] (9,5pt) -- ++ (0,-10pt);




    \node [align=right,anchor=east, color=Purple!90!black] at (0,-31.pt) {\scriptsize \baselineskip=8pt Incremental:\par};

    \node [align=right,anchor=east, color=Peach!95!black] at (0,-42pt) {\scriptsize \baselineskip=8pt One-shot:\par};

    \draw[draw=none, fill=Peach, fill opacity=0.2] (0,-46pt) rectangle ++(9, 6pt);
    \draw[draw=none, fill=Purple, fill opacity=0.2] (0,-35pt) rectangle ++(4.5, 6pt);
    \draw[draw=none, fill=Purple, fill opacity=0.2] (6.5,-35pt) rectangle ++(2.5, 6pt);

    \draw[draw=Purple, thick] (0,-35pt) -- ++(0, 6pt);
    \draw[draw=Purple, thick] (2,-35pt) -- ++(0, 6pt);
    \draw[draw=Purple, thick] (4,-35pt) -- ++(0, 6pt);
    \draw[draw=Purple, thick] (7,-35pt) -- ++(0, 6pt);
    \draw[draw=Purple, thick] (9,-35pt) -- ++(0, 6pt);

    \node[draw=none] at (5.5, -28pt) {\textcolor{Purple}{$\dots$}};

    \draw[draw=Peach, thick] (0,-46pt) -- ++(0, 6pt);
    \draw[draw=Peach, thick] (9,-46pt) -- ++(0, 6pt);
    
    \draw[->, thick, color=NavyBlue] (-2,31pt) node [above] {$\rw_0$} -- (-2,21pt);
    \draw[->, thick, color=NavyBlue] (0,31pt) node [above] {$\rw_n$} -- (0,21pt);
    \draw[->, thick, color=NavyBlue] (9,31pt) node [above] {$\rw_{n+1}$} -- (9,21pt);
    %
    %
    %
\end{tikzpicture}
\end{scaletikzpicturetowidth}
\end{center}
    \caption{Operational granularity of action selection.}
    \label{fig:slots}
\end{figure}
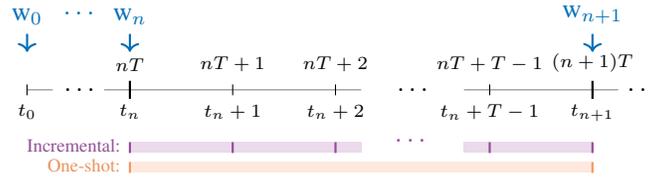

Each action is selected based on the state of the stochastic process describing the \gls*{eh} dynamics and the \gls*{es} level. Moreover, should the system considers a feedback information (\gls*{iaw}), prediction confidence can be accounted in the action selection. In the following sections we provide a detailed description of the mathematical models of each component.

\subsection{Adaptive \gls*{dnn} Model}
We define an adaptive \gls*{dnn} as $f(w; \theta)$, where $w\in\mathbb{R}^q$ and $\theta\in\mathbb{R}^d$ are the input and the trainable parameter vectors, respectively.
Without loss of generality, $f$ is considered as the composition of $L$ differentiable operators $l_i,\ i=1,\dots,L$, e.g., $l_i$ represents a convolutional layer in a \gls*{dnn}. The final output is denoted as \mbox{$\hat{y} = f(w;\theta) $}. In a multi-exit \gls*{dnn}, let $K\leq L$ be the number of exit branches added to the NN structure. In order to provide a valid prediction at each exit branch, task-specific classifiers with comparatively negligible processing cost 
are required.
We define the sub-network from the input layer to the $k$-th exit branch as $f(\cdot;\theta^{(k)})$, where $\theta^{(k)}$ represents the set of \gls*{dnn} parameters involved to the computation up to the $k$-th exit. Hence, $\hat{y}^{(k)}\triangleq f(w;\theta^{(k)})$ is the output at the $k$-th exit, and the final output is $\hat{y}=f(w;\theta^{(K-1)})=\hat{y}^{(K-1)}$.


Similarly, in the \gls*{mms} scenario, $K$ refers to the different models $f_k(\cdot; \theta^{(k)})$, $k=0,1,\dots,K-1$ to be independently trained and deployed at the device. At the beginning of the inference process, the controller chooses one of the models to execute.
We refer to each exit branch or model as a \textit{computing mode}. Associated with the $k$-th mode is a fixed processing energy cost $u\big( f(\cdot;\theta^{(k)}) \big)$, which, with a slight abuse of notation, is denoted as $u(k)$. Hence, processing modes that involve more parameters are more costly to perform, yet they output more reliable and accurate predictions. In this paper, we consider the $0$-th computing mode as a random-guesser that outputs energy-free random predictions. 

\subsection{Energy Provision Model}\label{sec:energy_provision_model}
We characterize the \gls*{eh} as a Markov-modulated process where $\rh_t\in\mathcal{H}\triangleq \{h_1, \dots, h_{|\mathcal{H}|}\}$ describes the environment states, with transition probabilities $p_{\rh}^{i,j}\triangleq \mathbb{P}(\rh_{t+1}=h_j|\rh_t=h_i)$.
The energy $\re^H_t\in\mathcal{E}^H=\{e^H_1,\dots,e^H_{|\mathcal{E}_H|}\}$ provided by the harvesting circuit depends on environment state $\rh_t$. Due to the uncertainty of environmental conditions, we model $\re^H_t$ as a discrete random variable with \gls*{pmf} $p_{\re^H}^{h}(e)\triangleq\mathbb{P}(\re^H_t=e|\rh_t=h)$, $e\in\mathcal{E}^H, h\in\mathcal{H}$.
The harvested energy $\re^H_t$ is used for the current inference task. Should $\re^H_t$ exceed the computation energy demand $\ru_t$, the excess is stored in the \gls*{es}, e.g., battery or supercapacitor. Conversely, if $\re^H_t$ is insufficient for the task's demands, the deficit is compensated by drawing the necessary additional energy from the \gls*{es}, when available. The \gls*{es} is modeled as a discrete buffer of energy packets with finite capacity $b_{\text{max}}$ \cite{opt_strategies_remote, michelusi, 4480065}. The energy level at time $t$ is denoted by $\rb_t=\{0,1, \dots, b_{\text{max}}\},$ and evolves according to 
\begin{align}\label{eq:evolution_battery}
    \rb_{t+1} = \min\{[\rb_t - \ru_t + \re^{H}_t]^+, b_{\text{max}}\}.
\end{align}
\begin{figure*}
    \centering
    \begin{subfigure}[b]{0.24\textwidth}
         \centering
         \includegraphics[scale=0.39]{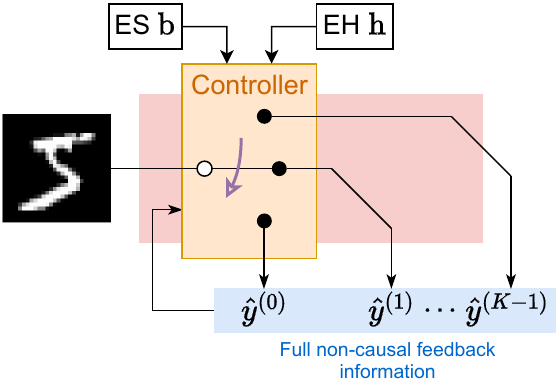}
         \caption{}
         \label{fig:4a}
     \end{subfigure}
     \hfill
     \begin{subfigure}[b]{0.24\textwidth}
         \centering
         \includegraphics[scale=0.39]{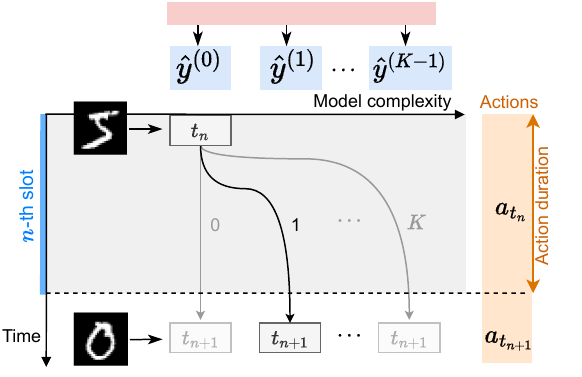}
         \caption{}
         \label{fig:4b}
     \end{subfigure}
     \hfill
     \begin{subfigure}[b]{0.24\textwidth}
         \centering
         \includegraphics[scale=0.39]{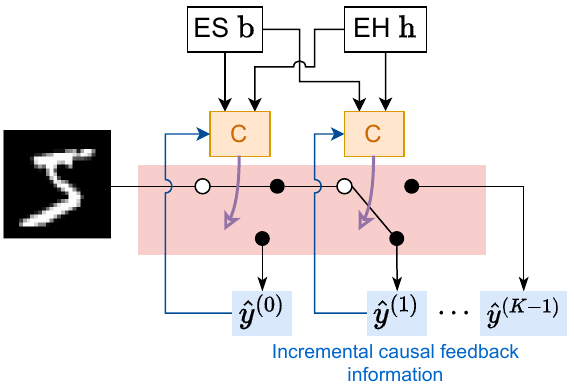}
         \caption{}
         \label{fig:4c}
     \end{subfigure}
     \hfill
     \begin{subfigure}[b]{0.24\textwidth}
         \centering
         \includegraphics[scale=0.39]{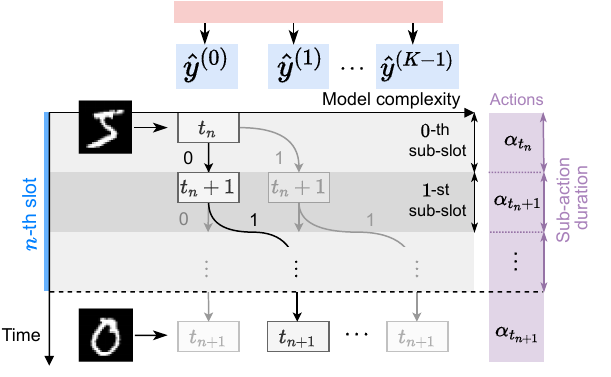}
         \caption{}
         \label{fig:4d}
     \end{subfigure}
    \caption{Controlled computing modules defined by the granularity of action selection and the availability of feedback information. (a) and (b) depict the control scheme and an the temporal dynamics of an oracle \gls*{os}-\gls*{iaw} controller, respectively. In the absence of feedback, this model reduces to the \gls*{mms} scheme. Conversely, (c) and (d) represent \gls*{inc} controls applicable to multi-exit networks: under causal feedback, the \gls*{inc}-\gls*{iaw}-\gls*{ee} scheme is realized, whereas the absence of feedback yields the \gls*{inc}-\gls*{iag}-\gls*{ee} scheme.}
    \label{fig:control_schemes}
\end{figure*}


\subsection{Instance-Aware Schemes}
Instance-aware methods integrates the prediction confidence of an adaptive \gls*{dnn} as feedback within a closed-loop control scheme. We first describe an oracle \gls*{os}-\gls*{iaw} controller having full information of the per-instance realizations of exit confidences. As mentioned in Section~\ref{sec:intro}, the theoretical derivations for \gls*{os}-\gls*{iaw} fit both \gls*{ee} and \gls*{mms}, since, mathematically, these two schemes are equivalent.

\subsubsection{One-Shot Instance-Aware (OsIAw) Controller (Oracle)}\label{sec:oneshot_appproach}
The one-shot controller operates every $T$ slots, that is whenever a new instance $\rw$ is provided by the sensor apparatus. Assuming that each \gls*{ee} step takes one time slot, it follows that $T\geq K-1$. For $T\gg K-1$, the time available for decision-making, thus for \gls*{eh}, is significantly longer than the time required for the most computationally intensive mode. As a result, the system experiences fewer constraints. Conversely, when $T=K-1$ the system is more constrained, as the time available for decision-making is reduced, leading to diminished \gls*{eh} within the decision epoch.

We model the problem as a discrete-time \gls*{mdp} $\langle \mathcal{X}, \mathbb{P}, \mathcal{A}, r \rangle$, with $t\in\mathbb{N}$ representing the time slot index, and $\mathcal{X}, \mathbb{P}, \mathcal{A}, r$ being the state space, the transition kernel, the action space and the reward, respectively. The decision epochs occur at time index $t_n$, and the controller adapts the computing mode (\gls*{dnn} model or exit branch) $\ra_{t_n}\in\mathcal{A}\triangleq\{0,\dots,K-1\}$. 


At time $t_n$, the state of the controller is defined as 
\begin{equation}\label{eq:oneshot_state}
    \rx_{t_n} = (\rvs_{t_n}, \rvz_{t_n}),
\end{equation}
where $\rvs_{t_n} \triangleq (\rb_{t_n}, \rh_{t_n})$ is the joint energy level and environment state, with $\mathcal{S}=\mathcal{B}\times\mathcal{H}$, and $\rvz_{t_n}$ encodes information about the \gls*{dnn} prediction quality for the current input instance. Ideally, $\rvz_{t_n}\in\{0,1\}^K$ denotes the correctness vector, where index $i$ contains a $1$ if the prediction at the corresponding computing mode is correct. However, at inference time the true label $\ry_{t_n}$ is unknown. Therefore, we define  $\rvz_{t_n}\in [0,1]^K$ as a measure of the vector of \textit{correctness likelihoods} of \gls*{dnn} predictions, i.e., confidence levels. Precisely, our objective is for the $i$-th component of $\rvz_{t_n}$ to represent the probability that the output prediction at the $i$-th computing mode is correct, given that its correctness likelihood estimate $g^{(i)}(\rw_{t_n})$ is $p$, i.e.,
\begin{align}
    \rvz^{(i)}_{t_n}=\mathbb{P}\Big(\ry_{t_n}=\hat{\ry}^{(i)}_{t_n}\mid g^{(i)}(\rw_{t_n})=p\Big).
\end{align}
When the predictive module is a \gls*{dnn}, $g^{(i)}(\rw_{t_n})=\max_j f_j(\rw_{t_n};\theta^{(i)})$, where $f_j(\rw_{t_n};\theta^{(i)})$ is the $j$-th component of the softmax output of the \gls*{dnn}, and $\hat{\ry}^{(i)}_{t_n}=\arg\max_j f_j(\rw_{t_n};\theta^{(i)})$ is the predicted label. Moreover, if the \gls*{dnn} is perfectly calibrated, then $\rz^{(i)}_{t_n}=g^{(i)}(\rw_{t_n})$ \cite{guo2017calibration,minderer2021revisiting}. Therefore, we will adopt the per-instance confidences of a calibrated model as an estimate measure of the \textit{correctness likelihoods} of \gls*{dnn} predictions.
We assume that $\rvz_{t_n}$ possesses a joint \gls*{pdf} $p_\rvz(z)$, $\vz\in [0, 1]^K$, and that $\{\rvz_{t_n}\}_{n\geq 0}$ are independent and identically distributed (iid). Note that the latter assumption is a direct consequence of assuming that $\{\rw_{t_n}\}_{n\geq 0}$ are iid, which is typical in machine learning applications for in-distribution instances.

Again, knowing $\rz_t$ for each new sample at the beginning of the slot requires the controller to have prior access to the confidences at all \gls*{dnn} computing modes, without paying the energy cost of running them. This is clearly unfeasible, thus we refer to such a controller as an \textit{oracle}.

Note that at each decision epoch, the controller has access to $\rb_{t_n}$ and $\rh_{t_n}$, but the value of $\re^H_{t}$ is uncertain for future time instances $t=t_n, t_n+1, \dots, t_{n+1}-1$. Therefore, the controller lacks foresight regarding potential energy outages, and the action selection process must be constrained to guarantee $\ru_{t_n}\leq \rb_{t_n}$ and assure continuous operations. Formally, at each state $\rx_{t_n}$ the feasible action space is $\mathcal{A}_{\rx_{t_n}}\triangleq \{a\in\mathcal{A}: u(a)\leq \rb_{t_n} \}$,
where $u(a)$ denotes the energy cost of $a$.
For example, if the \gls*{es} depletes ($\rb_{t_n}=0$), then only an energy-free random guess can be performed ($\ra_{t_n}=0$), regardless of the harvested energy in the current slot. In general, the smaller $b_{\text{max}}$, the more constrained the device is in its action selection process. Figures~\ref{fig:control_schemes}(a) and \ref{fig:control_schemes}(b) illustrate the \gls*{os} decision-making process with full non-causal feedback information.

The goal of a neural harvesting system is to continuously perform \gls*{dnn}-based inference under uncertain energy dynamics. When adaptive inference mechanisms (e.g., \gls*{ee}) are available, the goal is to accomplish inference tasks as accurately as possible while matching the energy constraints. To capture this behaviour, it is important that the reward encodes information on the performance of the system. Therefore, we define the reward function $r:\mathcal{S}\times\mathcal{Z}\times\mathcal{A}\to [0,1]$ for the one-shot controller as the $\ra_{t_n}$-th component of the likelihood vector; that is, the correctness likelihood of the selected computing mode: 
\begin{align}\label{eq:oneshot_reward}
    r(\rvs_{t_n}, \rvz_{t_n}, \ra_{t_n}) = \rvz^{(\ra_{t_n})}_{t_n}.
\end{align}


\subsubsection{Incremental Instance-Aware (IncIAw) \gls*{ee} Controller}
To fully exploit the intrinsic incremental nature of \gls*{ee} while containing the limitations in the action selection of the one-shot oracle controller, we study a sequential decision scenario modeled as a discrete-time \gls*{mdp} $\langle \mathcal{X}_{\text{inc}}, \mathbb{P}_{\text{inc}}, \mathcal{A}^{\text{inc}}, r_{\text{inc}} \rangle$, where actions $\ra_{t_n}$ are intended as incremental compositions of sub-actions.

The incremental controller operates in each slot $t$, where it chooses
a sub-action $\alpha_{t}$ to decide whether to \textit{pause} the computation and switching to idle mode, $\alpha_{t}=0$, or \textit{proceed} with the computation of the next exit,  $\alpha_{t}=1$. 
We denote the sub-action space as $\mathcal{A}^{\text{inc}}\triangleq\{0,1\}$. The decision process starts with an energy-free random prediction (exit 0), and proceeds by incrementally choosing sub-actions $\alpha_t$ in subsequent slots. Every $T$ slots, the computation for the current instance $\rw_{t_n}$ terminates, and the system becomes ready for a new sample $\rw_{t_{n+1}}$. For example, in a \gls*{dnn} implementing $K=3$ exits, setting $\alpha_{t_n}=0$ corresponds to performing an energy-free random prediction in slot $t_{n}$. With $\alpha_{t_n+1}=1$, the input instance is processed up to the first exit in the following slot $t_{n}+1$, resulting in $\ra_{t_n}=2$. Figures~\ref{fig:control_schemes}(c) and \ref{fig:control_schemes}(d)  illustrate the incremental decision-making process with causal feedback information.

At time $t$, the state of the controller is defined as 
\begin{equation}
    \rvx_t = (\rb_t, \rh_t, \xi_t, \tau_t, \rvz_t),
\end{equation}
with $\rb_t,\rh_t$ being the \gls*{es} and the \gls*{eh} processing states, $\xi_t=0,\dots,K-1$ representing the index of the current exit, $\tau_t = 0,\dots, T-1$ denoting the processing stage of the current input, that is $\tau_t=t\bmod T$, and $\rz_t\in[0,1]$ representing the correctness likelihood of the current exit $\xi_t$. This can be expressed as the $\xi_t$-th component of the vector of correctness likelihoods introduced in Section~\ref{sec:oneshot_appproach}, that is $\rz_t=\rvz_{t_n}^{(\xi_t)}$. Therefore, conversely to the one-shot formulation, whereby the whole likelihood vector is known, in the incremental approach it can be revealed progressively as the computation proceeds, depending on the sub-actions selected. 
This makes the incremental approach feasible since the needed information is causal. As before, $\rvs_t=(\rb_t, \rh_t, \xi_t, \tau_t)\in\mathcal{S}_{\text{inc}}$ represents the discrete component of the state, and we refer to the set $\{t_n, t_n+1, \dots, t_{n}+T-1\}$ as the processing stage of the $n$-th input instance.

Although at the beginning of the $n$-th processing stage, the value of $\re^H_{t_n}$ is still uncertain, intermediate realizations $\rb_{t}$ and $\rh_t$, $t=t_n+1, \dots, t_{n}+K-2$ can be observed, perhaps improving the estimation of future \gls*{eh} events and correcting the action selection on-the-go. This provides the controller with predictive insight into potential energy outages and overflow, enabling the execution of the $i$-th exit even when its total energy required 
 to produce an output exceeds the amount of energy available when the computation started, i.e. $\rb_{t_n}$. For example, in a system with $K=3$ exits and empty \gls*{es}, i.e., $b_{t_n}=0$, the one-shot controller is forced to select a random prediction. However, the incremental controller can switch to idle in the first slot and, in the case of sufficient energy provision, it can proceed further with the computation, potentially improving the quality of its prediction. Therefore, in each state $\rvx_{t}$ the feasible action space is $\mathcal{A}^{\text{inc}}_{\rvx_{t}}\triangleq \{\alpha\in\mathcal{A}^{\text{inc}}: u_{\text{inc}}(\alpha, \rvx_t)\leq \rb_{t}\}$, where $u_{\text{inc}}(\alpha, \rvx_t)$ is the energy cost of $\alpha$ in state $\rvx_t$, specifically at exit $\xi_t$.

Similarly to the one-shot approach, the reward function $r_{\text{inc}}:\mathcal{S}_{\text{inc}}\times[0,1]\times\mathcal{A}^{\text{inc}}\to [0,1]$ is defined as 
\begin{align*}
	r_{\text{inc}}(\rvs_{t}, \rz_{t}, \alpha_t) = \begin{cases}
	0 & \text{for }\tau_t=0,1,\ldots,T-2, \\ \rvz^{(\xi_t+\alpha_t)}_{i(t)} & \text{if } \tau_t=T-1
\end{cases},
\end{align*}
where $i(t)=\lceil\frac{t}{K-1}\rceil$, with $\lceil\cdot\rceil$ representing the ceil rounding operator. In words, a null reward is collected for intermediate steps. At final slot $t_n-1$, the reward is non-zero at the end of the $n$-th processing stage, where a decision is made by the controller.


\subsection{Instance-Agnostic (IAg) Schemes}
To asses the importance of taking actions based on the per-instance confidence values, we design both \gls*{os}- and \gls*{inc}-\gls*{iag} controllers. These controllers operate over the identical state space as their \gls*{iaw} counterparts, but disregard the per-instance confidence values. We model the systems as discrete-time \glspl*{mdp}, where the state and action spaces are $\mathcal{S}$, $\mathcal{A}$ and $\mathcal{S}^{\text{inc}}$, $\mathcal{A}^{\text{inc}}$, respectively. The performance of the $k$-th computing mode is measured by its prediction accuracy $\rho^{(k)}$, computed over a test dataset $\mathcal{D}$ as
\begin{align}
    \rho^{(k)} \triangleq \frac{1}{|\mathcal{D}|}\sum_{i=1}^{|\mathcal{D}|}\mathds{1}_{\{\hat{y}^{(k)}_i=y_i\}},
\end{align}
where $\hat{y}^{(k)}_i$ and $y_i$ denote the predicted label produced by the $k$-th mode, and the ground truth label, respectively, for the $i$-th input sample in $\mathcal{D}$. 
The rewards observed when action $a$, or sub-action $\alpha$, is selected in state $s$, respectively, are:  
\begin{align}
    r(\vs,a) &= \rho^{(a)},\\
    r_{\text{inc}}(\vs,\alpha) &= \begin{cases}
	0 & \text{for }\tau=0,1,\ldots,T-2, \\ \rho^{(\xi+\alpha)} & \text{if } \tau=T-1
    \end{cases}
\end{align}
In words, when the $k$-th exit classifier is selected, 
the reward observed by the controller is the accuracy of the $k$-th classifier. Hence, \gls*{iag} schemes use instance information from $\mathcal{D}$ to compute $\rho^{(k)}$, but this estimate is fixed, and does not change as a function of the confidence of the current input instance, $\rvz_t$. Note that, when the \gls*{dnn} is perfectly calibrated, the reward received by selecting the $k$-th mode matches the mean confidence level of that classifier, that is $\rho^{(a)}=(1/|\mathcal{D}|)\sum_{i=1}^{|\mathcal{D}|}\vz_i^{(a)}$.




\subsection{Energy Considerations}
 Let $E^{(i)}$ denote the cumulative energy required for processing an input instance up to the $i$-th exit.
Therefore, $E^{(1)}\leq E^{(2)}\leq\dots\leq E^{(K)}$.
When an \gls*{ee} is selected, i.e., $\ra_{t_n}\neq K$, the task's energetic demand decreases, $\ru_{t_n}<E^{(K)}$, and a reduced operating power
can be used to accomplish the inference task by the end of the slot. In the case of an energy outage, i.e., $\rb_{t_n} - E^{(1)} + \sum_{\tau=0}^{T-1}\re^H_{t_n+\tau}< 0$,
none of the \gls*{dnn} executions are feasible, forcing an energy-free random guess of the current instance label. Conversely, in situations of energy overflow, $\rb_{t_n}- \ru_{t_n} + \sum_{\tau=0}^{T-1}\re^H_{t_n+\tau}> b_{\text{max}}$,
the finite storage capacity of the \gls*{es} prevents further energy accumulation, causing potential lost opportunities for future \gls*{dnn} executions.

\section{System Optimization}\label{sec:optimization}


Given an initial state $x_0$, the goal is to find a one-shot stationary policy $\pi:\mathcal{S}\times\mathcal{Z}\times \mathcal{A}\to\mathscr{P}(\mathcal{A})$ and an incremental stationary policy $\pi_{\text{inc}}:\mathcal{S}_{\text{inc}}\times[0,1]\times \mathcal{A}^{\text{inc}}\to\mathscr{P}(\mathcal{A}^{\text{inc}})$, which maximize the infinite-horizon discounted reward, that is 
\begin{align}
	\pi^*(\vx) &= \argmax_{\pi\in\Pi} \mathbb{E}_\pi\Bigg[\sum_{n=0}^\infty \gamma^{n} r(\rvx_{t_n},\ra_{t_n}) \Big\lvert \rx_0=x\Bigg],\\
	\pi_{\text{inc}}^*(\vx) &	= \argmax_{\pi_{\text{inc}}\in\Pi_{\text{inc}}} \mathbb{E}_{\pi_{\text{inc}}}\Bigg[\sum_{t=0}^\infty \gamma_{\text{inc}}^{t} r_{\text{inc}}(\rvx_t,\alpha_t) \Big\lvert \rx_0=\vx\Bigg],
\end{align}
where $\gamma$ is a discount factor. Specifically, in order for $\pi^*$ and $\pi_{\text{inc}}^*$ to be comparable, it must holds that $\gamma_{\text{inc}}^{t_{n+1}-1}=\gamma^{n}$ for all $n\in\mathbb{Z}_{\geq1}$. This is because $r_{\text{inc}}$ is non-zero only when $\tau_t=T-1$, that is $t=t_n+T-1=t_{n+1}-1$, $\forall n\geq0$.


The following lemma shows the connection between the one-shot and the incremental approaches.
\begin{lemma}\label{lem:incr_corr}
     For any input instance $\rw_{t_n}$, and any specified action $a_{t_n}$, there exists a sequence of $K-1$ sub-actions $\{\alpha_{t_n+\tau}\}_{\tau=0}^{T-1}$ that would corresponds to $\ra_{t_n}\in\mathcal{A}$. Such correspondence can be expressed in closed-form as
     \begin{align*}
         \ra_{t_n}=\sum_{\tau=0}^{T-1}\alpha_{t_n+\tau}.
     \end{align*}
     Therefore, an optimal policy $\pi_{\text{inc}}^*:\mathcal{S}\times \Lambda\to\mathscr{P}(\Lambda)$ with respect to a generic objective in the incremental formulation is at least as good as an optimal policy for the same objective in the one-shot formulation $\pi^*:\mathcal{S}\times \mathcal{A}\to\mathscr{P}(\mathcal{A})$.
\end{lemma}
As a consequence, 
the incremental approach offers a more fine-grained control strategy, ensuring performance at least as good as the one-shot alternative. 
In the next section, we characterize the structure of optimal policies for the \gls*{os}-\gls*{iaw} controller and \gls*{mms}, and we describe the \gls*{dqn}-based optimization of the \gls*{inc}-\gls*{iaw}-\gls*{ee} controller.

\subsection{Optimal Policy for the \gls*{os}-\gls*{iaw} controller}\label{sec:oneshot_opt_policy}
In order to characterize the structure of the optimal policies we recall the definition of value-function and Q-function. The Bellman optimality equation for the svalue function is defined as 
\begin{align*}
        v^*(\vx) &= \max_{a\in\mathcal{A}_x}\mathbb{E}_{\rx'}\big[r(\vx,a)+\gamma v^*(\rx')\big],
\end{align*}
where $v^*(\vx)\triangleq \max_{a\in\mathcal{A}_\vx}q^*(\vx,a)$, and
\begin{align}\label{eq:q_function}
	q^*(\vx,a) \triangleq \mathbb{E}_{\rx'}\big[r(\vx,a) + \gamma \max_{a'\in\mathcal{A}_\vx} q^*(\rx',a')\big].
\end{align}

\begin{theorem}[Optimal value function]\label{th:pwl_vf}
The optimal value function for the \gls*{os}-\gls*{iaw} controller is piece-wise linear in $\vz\in\mathcal{Z}$
\begin{align*}
    v^*(\vs,\vz) &= \sum_{j=0}^{K-1} \Big(e^\top_j\vz + \gamma \sum_{\vs'\in\mathcal{S}} \mathbb{P}(\vs'|\vs,a_j)\bar{v}^*(\vs')\Big)\mathds{1}_{\{\vz\in\mathcal{Z}_j(\vs)\}},
\end{align*}
where $\bar{v}^*(\vs)\triangleq \mathbb{E}_{\rvz}[v^*(\vs,\rvz)]$, and $\mathcal{P}_\vs\triangleq\{\mathcal{Z}_j(\vs)\}_{j=0}^{K-1}$ forms a partition of $\mathcal{Z}$ for each state $\vs\in\mathcal{S}$. Moreover, the partition $ \mathcal{P}_s$ can be parameterized by at most $K-1$ real values, and can be expressed as
\begin{align*}
    \mathcal{Z}_j(\vs) = \{\vz\in\mathcal{Z}: \mM_j \vz\geq\mF_j\vdelta(\vs)\},\quad j=0,\dots,K-1,
\end{align*}
where $\vdelta(\vs)=[\delta_{0i}, i=1,\dots,K-1]\in\mathbb{R}^{K-1}$ is a threshold vector with $\delta_{0i}=\gamma\big(\mathbb{P}_{a_0}(s) - \mathbb{P}_{a_i}(\vs)\big)\bar{v}^*$, and $\mF_j\in\{0,\pm1\}^{K-1\times K-1}$, $\mM_j\in\{0,\pm 1\}^{K-1\times K}$ are appropriate matrices.
\end{theorem}
As direct consequence of Theorem~\ref{th:pwl_vf} we can state the following lemma.
\begin{lemma}[Optimal policy]
    The optimal policy $\pi^*$ for the one-shot formulation is
    \begin{align}
        \pi^*(\vs,\vz) &=\sum_{j=0}^{|\mathcal{A}_\vs|-1} a_j\mathds{1}_{\{\vz\in\mathcal{Z}_j(\vs)\}},\quad a_j\in\mathcal{A}_\vs,
    \end{align}
    where $\mathds{1}_{\{D\}}$ is the indicator function for event $D$.
\end{lemma}
The previous lemmas characterize the structure of the optimal value function, establishing that the one-shot optimal policy requires $|\mathcal{S}|$ partitions of $\mathcal{Z}$, $\mathcal{P}_s$, each parameterized by $K-1$ values, say thresholds. This determines a finite structure, which is essential to design a value-iteration-based algorithm to compute $\varepsilon$-optimal policies. In fact, the algorithm can impose the optimal structure and iterate over a restricted set of policies, which contains the optimal ones. Such a structure simplifies when the $0$-th exit classifier is a random predictor which outputs a random label for the current input instance. In this case, the $0$-th entry of the confidence vector is constant, that is $z^{(0)}=1/|\mathcal{Y}|$, where $|\mathcal{Y}|$ is the number of classes. This simplifies the structure of the optimal policy, providing a simpler geometrical interpretation.

\begin{example}[\gls*{os}-\gls*{iaw} with Initial Random Predictor]
    Consider a \gls*{dnn} model implementing $K=3$ computing modes, where the $0$-th is a random predictor. The action space and the confidence vector are $\mathcal{A}=\{0,1,2\}$ and $\vz=[1/|\mathcal{Y}|, z^{(1)}, z^{(2)}]$, respectively. In each state $\vs=(b,h)$, the optimal value function is partitioned into $\{\mathcal{Z}_0(\vs), \mathcal{Z}_1(\vs), \mathcal{Z}_2(\vs)\}$, where $\vdelta(\vs)=[\delta_{01}(s), \delta_{02}(\vs)]$ and
    \begin{align}
        \begin{split}
            \mM_0&=\begin{bmatrix}
                1 & -1 & 0\\
                1 & 0 & -1
            \end{bmatrix},\ 
            \mM_1=\begin{bmatrix}
                -1 & 1 & 0\\
                0 & 1 & -1
            \end{bmatrix},\
            \mM_2=\begin{bmatrix}
                -1 & 0 & 1\\
                0 & -1 & 1
            \end{bmatrix},\\
            \mF_0&=\begin{bmatrix}
                -1 & 0\\
                0 & -1
            \end{bmatrix},\ 
            \mF_1=\begin{bmatrix}
                1 & 0\\
                1 & -1
            \end{bmatrix},\ 
            \mF_2=\begin{bmatrix}
                0 & 1\\
                -1 & 1
            \end{bmatrix}.
        \end{split}
    \end{align}
    In particular, the equations describing $\mathcal{Z}_0(\vs)$,
    \begin{align}
    z^{(1)}\leq \frac{1}{|\mathcal{Y}|} + \delta_{01}(\vs),\ \textrm{and } z^{(2)}\leq \frac{1}{|\mathcal{Y}|} + \delta_{02}(\vs)
    \end{align}
    defines a rectangular set in the $z^{(1)}$-$z^{(2)}$ plane, while for $\mathcal{Z}_1(\vs)$ and $\mathcal{Z}_2(\vs)$ we have
    \begin{align}
        \mathcal{Z}_1(\vs):
        &\begin{cases}
            z^{(1)}\geq \frac{1}{|\mathcal{Y}|} + \delta_{01}(\vs)\\
            z^{(1)}\geq z^{(2)} - \delta_{02}(\vs) +  \delta_{01}(\vs),
        \end{cases},\\
        \mathcal{Z}_2(\vs):
        &\begin{cases}
            z^{(2)}\geq \frac{1}{|\mathcal{Y}|} + \delta_{02}(\vs)\\
            z^{(1)}\leq z^{(2)} - \delta_{02}(\vs) +  \delta_{01}(\vs)
        \end{cases},
    \end{align}
    corresponding to a $45^{\circ}$ separating line in the $z^{(1)}$-$z^{(2)}$ plane. Figure~\ref{fig:opt_vf} visually represents the partition $\mathcal{P}_\vs$ in $z^{(1)}$-$z^{(2)}$. 
    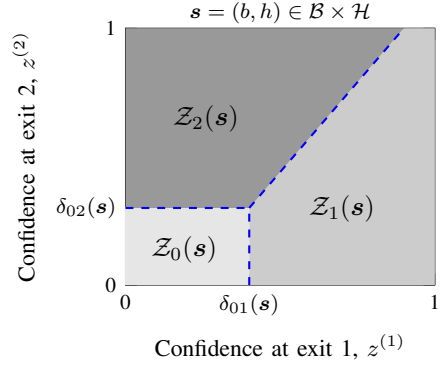
\begin{figure}
        \centering
            \begin{tikzpicture}
        \begin{axis}[
            title=\footnotesize{$\vs=(b,h)\in\mathcal{B}\times \mathcal{H}$},
            title style={at={(0.5,.92)}},
            ticklabel style={font=\footnotesize},
            xlabel= {\small Confidence at exit 1, $z^{(1)}$},
            ylabel={\small Confidence at exit 2, $z^{(2)}$},
            scale=0.6,
            xtick={0, 0.4, 1},
            xticklabels={0, $\delta_{01}(\vs)$, 1},
            ytick={0, 0.3, 1},
            yticklabels={0, $\delta_{02}(\vs)$, 1},
            ymin=0,
            ymax=1, 
            xmin=0, 
            xmax=1]
            \fill[gray!20] (0,0) -- (0.4,0.) -- (0.4,0.3) -- (0,0.3) -- cycle;
            \fill[gray!40] (0.4,0) -- (0.4,0.3) -- (0.9,1) -- (1,1) -- (1,0) -- cycle;
            \fill[gray!80] (0,0.3) -- (0.4,0.3) --(.9,1) -- (0,1) -- cycle;
            \addplot[mark=none, blue, dashed, thick] coordinates {(0.4,0.3) (0.9,1)};
            \addplot[mark=none, blue, dashed, thick] coordinates {(0.4,0) (0.4,0.3)};
            \addplot[mark=none, blue, dashed, thick] coordinates {(0.,0.3) (0.4,0.3)};
            \node[black] at (0.19,0.14) {$\mathcal{Z}_0(\vs)$};
            \node[black] at (0.7,0.3) {$\mathcal{Z}_1(\vs)$};
            \node[black] at (0.26,0.65) {$\mathcal{Z}_2(\vs)$};
        \end{axis}
    \end{tikzpicture}
        \caption{Partition $\mathcal{P}_\vs$ in the $z^{(1)}$-$z^{(2)}$ induced by the optimal value function with $K=3$ exit classifiers.}
        \label{fig:opt_vf}
    \end{figure}
    Therefore, in state $(\vs,\vz)$, the optimal policy $v^*$ selects action $a_j$ whenever $\vz\in\mathcal{Z}_j(\vs)$. 
\end{example}

In order to design an algorithm to find the optimal policy for the one shot controller, we restrict the search to policies with the structure described in Section~\ref{sec:oneshot_opt_policy}. Let
\begin{align*}
    \vDelta_{\mathbb{P}(\vs)}\triangleq [\mathbb{P}_{a_1}(\vs)-\mathbb{P}_{a_j}(\vs),\ j=1,\dots,K-1].
\end{align*}
Hence, $\vdelta^*(\vs)=\vDelta_{\mathbb{P}(\vs)}\bar{\vv}^*$, and we rewrite $\bar{v}^*(\vs) = \mathbb{E}_\rz[v^*(\vs,\rz)]$ as
\begin{align}
    \begin{split}
        \bar{v}^*(\vs) = \sum_{j=0}^{K-1}\mathbb{E}_{\rz}\Big[ \Big(e^\top_j\rz + \gamma \mathbb{P}_{a_j}(\vs)\bar{\vv}^*\Big)\sgn(\mM_j\rz-\mF_j(\vDelta_{\mathbb{P}(\vs)}\bar{\vv}^*))\Big],&
    \end{split}
\end{align}
noting that $\bar{v}^*(\vs)$ is the fixed point of the following optimality operator
\begin{align}
    \begin{split}
        (\mathcal{T}\bar{v})(\vs) = \sum_{j=0}^{K-1}\mathbb{E}_{\rz}\Big[ \Big(e^\top_j\rz + \gamma \mathbb{P}_{a_j}(\vs)\bar{\vv}\Big)\sgn(\mM_j\rz
        -\mF_j(\vDelta_{\mathbb{P}(\vs)}\bar{\vv}^*))\Big].&
    \end{split}
\end{align}

Let $\mathcal{T}^l$ be the composition of $\mathcal{T}$ with itself $l$ times. Since for discounted problems $\bar{v}(\vs) = \lim_{l\to\infty}(\mathcal{T}^l \bar{v})(\vs),\forall \vs$ holds \cite{bertsekas1996neuro}, \gls*{vi} can be used. However, applying $\mathcal{T}$ requires the controller to know the joint distribution $p_\rvz$, which is unknown. Therefore, at each iteration of \gls*{vi}, we approximate $\bar{v}(\vs)$ with the empirical-mean estimator $\hat{\bar{v}}(\vs)$ computed over an estimation set $\mathcal{D}_{\text{est}}=\{\vz_i\}_{i=1}^D$.

This leads to the definition of an approximate optimality operator $\hat{\mathcal{T}}$ as follows
\begin{align}
    \begin{split}
        (\hat{\mathcal{T}}\bar{v})(\vs) = \frac{1}{D}\sum_{i=1}^D\sum_{j=0}^{K-1}\Big(e^\top_j\vz_i + \gamma \mathbb{P}_{a_j}(\vs)\bar{\vv}\Big)\sgn(\mM_j\vz_n
        -\mF_j(\vDelta_{\mathbb{P}(\vs)}\bar{\vv}))\Big].&
    \end{split}
\end{align}

Algorithm \ref{alg:vi} described the $\epsilon$-approximate \gls*{vi} which results in an $\epsilon$-optimal $\bar{\vv}_\epsilon$ such that $||\bar{\vv}_\epsilon-\hat{\bar{\vv}}||_\infty\leq\epsilon$, whereby $\hat{\bar{\vv}}=\hat{\mathcal{T}}\hat{\bar{\vv}}$. Moreover, $\hat{\bar{\vv}}$ is an approximate solution of $\bar{\vv}=\mathcal{T}\bar{\vv}$, meaning that $\hat{\bar{\vv}}\approx \bar{\vv}^*$. Clearly, due to the strong law of large numbers, as the sample size $N\to\infty$, $\hat{\bar{\vv}}\to \bar{\vv}$ a.s., and an $\epsilon$-optimal policy can be found. Otherwise, for finite $N$, an $\epsilon$-sub-optimal policy will be obtained.

\begin{algorithm}[t]
    \caption{$\epsilon$-Approximate \gls*{vi}}\label{alg:vi}
    \begin{algorithmic}
        \Require: $l=0$, estimation set $\mathcal{D}_{\text{est}}$, $\bar{v}_0(\vs) = 0,\,\forall \vs$
        \Repeat 
        \State $\bar{v}_{l}(\vs) = (\hat{\mathcal{T}}\bar{v}_{l-1})(\vs),\ \forall \vs$
        \Until{$||\bar{\vv}_{l}-\bar{\vv}_{l-1}||_{\infty}\leq\epsilon$}
    \end{algorithmic}
\end{algorithm}


\subsection{Suboptimal Policy for the \gls*{inc}-\gls*{iaw}-\gls*{ee}}
The incremental approach partitions each slot into non-empty disjoint sub-slots, where the controller chooses a sub-action to decide whether to switch to \textit{idle mode}, or to \textit{proceed} with the computation of the next exit. Unlike the one-shot formulation where $z$ represents an uncontrollable state component, in the incremental scenario the evolution of $z$ can be directly affected by the action $\alpha$. Should the controller choose to proceed to the next exit, it will add further inference processing to the current input instance.
Since the input of each layer of the model is the output of the previous layer, the sequential nature exhibits a Markovian property, whereby the future state is solely dependent on the current state.
Therefore, given the state $(\vs,z)$ and the action $\alpha$, the next state $(\vs',z')$ is determined as follows: $\vs'$ is generated according to $\mathbb{P}(\vs'|\vs, \alpha)$ and the next confidence value $z'$ is generated according to the conditional probability $\mathbb{P}(z'|\vs, z,\alpha)$. Note that the latter is unknown and, for each $\vs$, it is a function of $z\in[0,1]$.
 We could construct a model to learn these conditional probabilities, and design a model-based controller. %
However, optimizing separately the learning and the control problem may face extra costs without providing proportional performance benefits to the whole system. Therefore, we use \gls*{dqn}~\cite{dqn} to optimize the control problem without the explicit need for the approximation of the conditional distributions of confidence values $\mathbb{P}(\cdot|\vs,z,\alpha)$.

\gls*{dqn} uses a \gls*{nn}, parameterized by $\theta$, to approximate the optimal action-value function for all possible actions within a given state, that is $q_\theta^*(\vx, a)\approx q^*(\vx, a)$. At each time step $t$ a replay buffer $R_{t}=\{e_{i}\}_{i=1}^t$ stores the agent's past experiences $e_{t}=(\vx_{t}, a_{t}, r(\vx_{t}, \alpha_{t}), \vx_{t+1})$ collected while interacting with the environment following an $\varepsilon$-greedy policy based on $q_{\theta_t}(\vx, a)$. During learning, mini-batches of past experiences are sampled uniformly at random from $R_{t}$, and $\theta_{t}$ is updated minimizing \cite{dqn}
\begin{align*}
    \begin{split}
         \mathbb{E}_{(\rvx,\alpha,\rvx')\sim R_{t}}\Big[\big( r(\rvx,\alpha) 
         +\gamma \max_{\alpha'} q_{\tilde{\theta}_{t}}(\rvx',\alpha') - q_{\theta_{t}}(\rvx,\alpha) \big)^2\Big]
    \end{split}
\end{align*}

through stochastic gradient descent. $\tilde{\theta}_{t}$ are the target network parameters which are periodically updated to mitigate correlations between the action and the target values.\\
Notably, the usage of a \gls*{dnn}-based approach to optimize the energy performance of another \gls*{dnn} may appear excessive. However, we use a lightweight fully-connected \gls*{dnn} to approximate the action-value function, described in Section \ref{sec:dqn_target}.

\subsection{Optimal Policy for the \gls*{mms} controller}
We characterize the optimal policy of \gls*{mms} controller.
\begin{theorem}[Optimality of monotone policies for \gls*{mms}]\label{th:mms_optimal_policy}
    For the \gls*{mms} problem, there exists a monotone non-decreasing optimal policy in the \gls*{es} level $b\in\mathcal{B}$ of the form
    \begin{align}\label{eq:opt_mms_policy}
        \pi_{\text{\gls*{mms}}}^{*}(b,h) = \min\{ a': a'\in\argmax_{a\in\mathcal{A}(b)} q_h^*(b,a)  \},
    \end{align}
    where $q^*_h(b,a)$ is the optimal Q-function for every $h\in\mathcal{H}$ as in \eqref{eq:q_function}.
\end{theorem}
Showing that the optimal policy for the \gls*{mms} is monotone in the \gls*{es} level has significant practical implications in developing efficient algorithms. Specifically, because the optimal solution reduces to identifying $K$ thresholds on the \gls*{es} level, a lookup table can be effectively stored on a resource-constrained device.

\section{Experimental Setup}\label{sec:exp_setup}
To validate our theoretical findings, we conducted simulation-based experiments. This section provides a comprehensive description of the architecture of \gls*{dqn} and the multi-exit \gls*{dnn}  employed, including a detailed analysis of per-layer \glspl*{flop}, which underpins our rationale for implementing early exits within the architecture. Furthermore, we elaborate on the training methodology and calibration procedures employed, alongside the model and parameters used for the energy provision framework.

\subsection{Dataset}
In our experiments, we use a multi-exit \gls*{dnn} for image classification on Tiny Imagenet \cite{le2015tiny}.
We partition the dataset into \(\mathcal{D}_{\text{train}}\), \(\mathcal{D}_{\text{cali}}\), \(\mathcal{D}_{\text{est}}\) and \(\mathcal{D}_{\text{test}}\) as follows: \(D_{\text{train}}\) receives 70\% of the data for the \gls*{dnn} training processes, \(\mathcal{D}_{\text{cali}}\) is allocated 10\% for calibration tasks, \(\mathcal{D}_{\text{est}}\) is assigned 10\% to empirically estimated $\hat{v}$ and learn the optimal policy, and the testing partition, \(\mathcal{D}_{\text{test}}\), receives the remaining 10\% of the data.

\subsection{Multi-Exit EfficientNet: Design, Training and Calibration}\label{sec:result_ee_design}
We design multi-exit EfficientNet architectures based on EfficientNet-(B0-B7) models \cite{Efficientnet}. The exit classifier, attached at the end of each stage, is devised by replicating the structure of the final classifier, while adapting the input feature map size to match the output size of each stage. Detailed specifications of these exit classifiers are provided in Table \ref{tab:exit_architecture}.

For each EfficientNet-B$m$, we construct seven sub-networks $f_i(\cdot; \theta^{(m)}),\ i=0,\dots,6$, each formed by the composition of the first $i$-th stages and the exit classifier. We measure the computational cost, in terms of \glspl*{flop} for $f_i(\cdot; \theta^{(m)})$ to process an input instance from the Tiny ImageNet dataset. Figure~\ref{fig:efficientnet_flops} illustrates the computational cost in terms of \glspl*{flop} for multi-exit Efficientnet-B0 through Efficientnet-B4 relative to the stage index after which the exit classifier is attached. Intuitively, the increasing complexity from EfficientNet-B0 to EfficientNet-B4 is a consequence of compound scaling \cite{Efficientnet}, which also results in higher \glspl*{flop}. Interestingly, for a given model $m$, the \glspl*{flop} do not always increase with the exit index $i$: due to the decreasing input resolution of each stage with $i$. At earlier stages, the exit classifier has to process higher input resolution, which offsets the increase in channels and the number of layers per stage. For instance, for EfficientNet-B0 through EfficientNet-B2, attaching an exit classifier at the first stage is impractical because the resulting sub-network is shallower and requires more \glspl*{flop} than its subsequent stage.

In this study, we employ EfficientNet-B2 as the backbone model, augmented with an initial energy-free random predictor. Additionally, we introduce two early exits after the $3$rd and $5$th stages, resulting in a 4 multi-exit neural network. This selection is informed by the observation that the \glspl*{flop} required by these two early exits align with those of EfficientNet-B0 and EfficientNet-B1, as indicated by the dashed horizontal lines in Figure~\ref{fig:efficientnet_flops}. Moreover, the cost in terms of \glspl*{flop} of the two early exits and the final one is approximately linear.
To train the multi-exit EfficientNet on Tiny ImageNet \cite{le2015tiny}, we formulate a joint optimization problem by aggregating the loss functions of the exit branches into a unified objective \cite{teerapittayanon2016branchynet}. We adopt a similar training process as in \cite{Efficientnet}, using Adam optimizer \cite{kingma2014adam} with a learning rate of $1$e-3, weight decay $1$e-5 and momentum $0.9$. The test set accuracies obtained by each exit classifier are $0.53$, $0.69$, $0.83$, respectively. We perform temperature scaling calibration \cite{guo2017calibration} to assure that the exit confidence of the classifiers is a representative estimate of the true likelihood. 


\begin{table}[]
    \centering
    \begin{tabular}{llcc}
    \toprule
        Network  & Operator & Channels\\
        \midrule
        \multirow{3}{4em}{Exit Classifier} & MBConv4, k3x3 & 304 \\
        & Conv1x1 & 608 \\
        & BN \& SiLU \& Pooling \& FC & 200 \\
        \midrule
        \multirow{3}{5em}{Deep Q-Network} & Fully Connected & 64 \\
        & Fully Connected & 64 \\
        & Fully Connected & 2 \\
    \end{tabular}
    \caption{Architecture of the exit classifier and \gls*{dqn}.}
    \label{tab:exit_architecture}
\end{table}


\begin{figure}
    \centering
     \usetikzlibrary{matrix}
\begin{tikzpicture}[remember picture]
    \begin{axis}[
        width=8.5cm,
        height=7.5cm,
        xlabel={Exit Anchors},
        ylabel={G\glspl*{flop}},
        xmin=-0.3,
        xmax=6.3,
        legend columns=5, 
    ]

        \coordinate (insetPosition) at (rel axis cs:0.01,0.95);
        \coordinate (legend) at (axis description cs:0.5,1.0);
    
        \addplot[Violet, thick, mark=*, mark size=1.5pt] table [x=x, y=efficientnet_b0, col sep=comma] {figures/results/flops/flops.csv};
        \addplot[mark=none, Violet, dashed, domain=0:6] {0.400635392};

        \addplot[Plum, thick, mark=*, mark size=1.5pt] table [x=x, y=efficientnet_b1, col sep=comma] {figures/results/flops/flops.csv};
        \addplot[mark=none, Plum, dashed, domain=0:6] {0.711708416};

        \addplot[NavyBlue, thick, mark=*, mark size=1.5pt] table [x=x, y=efficientnet_b2, col sep=comma] {figures/results/flops/flops.csv};
        \addplot[mark=none, NavyBlue, dashed, domain=0:6] {1.12502256};

        \addplot[CornflowerBlue, thick, mark=*, mark size=1.5pt] table [x=x, y=efficientnet_b3, col sep=comma] {figures/results/flops/flops.csv};
        \addplot[mark=none, CornflowerBlue, dashed, domain=0:6] {1.882674768};

        \addplot[Aquamarine, thick, mark=*, mark size=1.5pt] table [x=x, y=efficientnet_b4, col sep=comma] {figures/results/flops/flops.csv};
        \addplot[mark=none, Aquamarine, dashed, domain=3:6] {4.5110616};

    \end{axis}

    \begin{axis}[
        at={(insetPosition)},
        anchor={outer north west},
        x=0.4cm,
        y=0.1cm,
        footnotesize,
        legend columns=7,
        legend style={at={(0.,1.)}, anchor=north west}
        ]
        
        \addplot[Violet, thick, mark=*, mark size=0.6pt] table [x=x, y=efficientnet_b0, col sep=comma] {figures/results/flops/flops.csv};
        \label{plot:line0}

        \addplot[Plum, thick, mark=*, mark size=0.6pt] table [x=x, y=efficientnet_b1, col sep=comma] {figures/results/flops/flops.csv};
        \label{plot:line11}

        \addplot[NavyBlue, thick, mark=*, mark size=0.6pt] table [x=x, y=efficientnet_b2, col sep=comma] {figures/results/flops/flops.csv};
        \label{plot:line2}

        \addplot[CornflowerBlue, thick, mark=*, mark size=0.6pt] table [x=x, y=efficientnet_b3, col sep=comma] {figures/results/flops/flops.csv};
        \label{plot:line3}

        \addplot[Aquamarine, thick, mark=*, mark size=0.6pt] table [x=x, y=efficientnet_b4, col sep=comma] {figures/results/flops/flops.csv};
        \label{plot:line4}

        \addplot[Green, thick, mark=*, mark size=0.6pt] table [x=x, y=efficientnet_b5, col sep=comma] {figures/results/flops/flops.csv};
        \label{plot:line5}

        \addplot[LimeGreen, thick, mark=*, mark size=0.6pt] table [x=x, y=efficientnet_b6, col sep=comma] {figures/results/flops/flops.csv};
        \label{plot:line6}
        
    \end{axis}

    \matrix [
            matrix of nodes,
            anchor=south,
        ] at (legend) {
            \ref*{plot:line0} \footnotesize B0 & \footnotesize \ref*{plot:line11} B1 & \footnotesize \ref*{plot:line2} B2 & \footnotesize  \ref*{plot:line3} B3 & \footnotesize  \ref*{plot:line4} B4 & \footnotesize  \ref*{plot:line5} B5 & \footnotesize  \ref*{plot:line6} B6 \\
        };
\end{tikzpicture}
    \caption{\glspl*{flop} required to process an input instance for each sub-network $f_i(\cdot;\theta^{(i)})$, with $i$ being the exit anchor, of the corresponding EfficientNet model. The inset picture shows the \glspl*{flop} for the all EfficientNet models (from B0 to B6).
    }
    \label{fig:efficientnet_flops}
\end{figure}
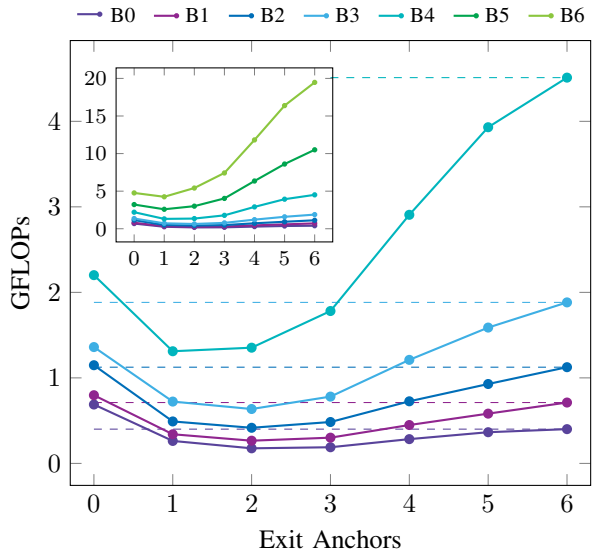

\begin{table}[t]
    \centering
    \begin{tabular}{ccccc}
    \toprule
        $p^G_{\re^H}$ & $p^B_{\re^H}$ &\multicolumn{1}{p{1.5cm}}{\centering Energy Rate $\mu$} & \multicolumn{1}{p{2.cm}}{\centering Calibrated Accuracy \%} & \multicolumn{1}{p{2.35cm}}{\centering Uncalibrated Accuracy \%} \\
       \midrule
        0.2 &  0.1  & 0.54 & \textbf{39.4} & 37.4 \\
        0.4 &  0.2  & 1.11 & \textbf{57.6} & 56.0 \\
        0.7 &  0.35 & 1.92 & \textbf{75.7} & 72.1 \\
        0.9 &  0.55 & 2.52  & \textbf{80.6} & 78.2 \\
        1   &  0.75 & 2.88 & \textbf{81.4} & 80.2 \\
        1   &  1    & 3.00   & \textbf{83.1} & 80.8 \\
    \end{tabular}
    \caption{The effect of calibration on the accuracy of OS-IAw controller. 
    We set $b_{\text{max}}=5$, $p_G=0.9$ and $p_B=0.5$.
    }
    \label{tab:cali}
\end{table}

\subsection{Deep Q-Network}\label{sec:dqn_target}
For the implementation of the incremental policy, the deep Q-network is designed as a compact fully connected \gls*{dnn} with two hidden layers of $64$ neurons each. The total number of \glspl*{flop} is $6.6$k, which represents only $0.0017\%$ of the computational complexity of EfficientNet-B0. The default hyperparameters are adopted from \cite{dqn}, except for the optimizer and learning rate  which are sourced from \cite{stable-baselines3}.

\subsection{Energy Harvesting (EH)}
As outlined in Section~\ref{sec:energy_provision_model}, we consider, without any loss of generality, a Markov chain with two environment states: ($G$)ood, and ($B$)ad, i.e., $\mathcal{H}=\{G,B\}$. These states signify favorable and unfavorable conditions, respectively. For instance, in the context of solar energy, the states could represent diurnal cycles of day and night or meteorological variations like sunny and cloudy conditions. Notice that the theoretical results derived in Section~\ref{sec:sys_model} hold for a generic Markov chain. Hence, a real \gls*{eh} source could be modeled better with additional environmental states. Nevertheless, it would also increase the complexity of the analysis unnecessarily.

For brevity, we define the transition probabilities as $p_\rh^{G} \triangleq p^{G,G}_\rh$ and $p^B_\rh \triangleq  p^{B,B}_\rh$, $\forall t$.
Similarly, the energy units harvested over a time slot are $\re_t^H\in\{0,1\}$ with $p^G_{\re^H}\triangleq p^G_{\re^H}(1)$ and $p^B_{\re^H}\triangleq p^B_{\re^H}(1)$, $\forall t$.
The incoming energy rate generated by the \gls*{eh} process is defined as
\begin{align}\label{eq:e_rate}
    \mu \triangleq p^\infty_G p^G_{\re^H} + p^\infty_B p^B_{\re^H},
\end{align}
where $ p^\infty_h$, $h\in\{G,B\}$, is the limiting distribution of the Markov chain $\{\rh_t\}$.


\section{Experimental Results}\label{sec:exp_results}
This section provides an extensive discussion of our experimental results. Section~\ref{sec:cali_vs_uncali_results} investigates the impact of calibration on the decision-making performance under different energy constraints. In Section~\ref{sec:policy_comparison}, we numerically derive and compare the policies for the controllers under study (listed in Table~\ref{tab:methods_studied}). Specifically, we utilize Algorithm~\ref{alg:vi} to numerically obtain the \(\epsilon\)-suboptimal policy for the \gls*{os}-\gls*{iaw} oracle controller. For comparison, we also derive a model-free\footnote{By model-free we mean that the controller does not require the knowledge of the transition probabilities of the environment.} oracle controller using \gls*{dqn}, referred to as \gls*{os}-\gls*{iaw}-EE-oracle-\gls*{dqn}.
The \gls*{mms} policy, which is \gls*{os}-\gls*{iag}, is computed using \gls*{pi}. To assess the advantages of \glsdesc*{inc} controllers, we employ \gls*{pi} to compute an 
\(\epsilon\)-optimal policy for the \gls*{inc}-\gls*{iag}-\gls*{ee}.
Ultimately, \gls*{dqn} is used to derive the \gls*{inc}-\gls*{iaw}-\gls*{ee}-\gls*{dqn} policy. In Section~\ref{sec:controller_comparison}, we conduct an extensive set of simulations to evaluate the long-term average accuracy (non-discounted) of the described models under different \gls*{es} and \gls*{eh} conditions. This extensive testing enables us to compare the performance of our multi-exit \gls*{dnn} architectures across a range of operational scenarios, providing a robust assessment of their effectiveness and efficiency.

\subsection{Calibrated vs Uncalibrated Oracle}\label{sec:cali_vs_uncali_results}
We study the effect of calibration on decision-making performance of the oracle controller, i.e., \gls*{os}-\gls*{iaw}. We select some representative values for the harvested energy \gls*{pmf}, where each corresponds to a different energy rate, therefore to a certain degree of energy constraint. We compute an \(\epsilon\)-optimal policy through Algorithm~\ref{alg:vi} using an uncalibrated and a post-training calibrated \gls*{dnn} model. We measure the performance by the average number of correctly classified input samples in \( \mathcal{D}_{\text{test}}\), computed for $200$ episodes of length $100$ steps. Table~\ref{tab:cali} reports the parameters selected and shows the dominance of the calibrated \gls*{dnn} over its uncalibrated counterpart, with long-term average accuracy improvements varying between $1\%$ and $3.6\%$. For this reason, in the remaining simulations we use calibrated \glspl*{dnn}.

\begin{figure*}[ht]
    \centering
    \input{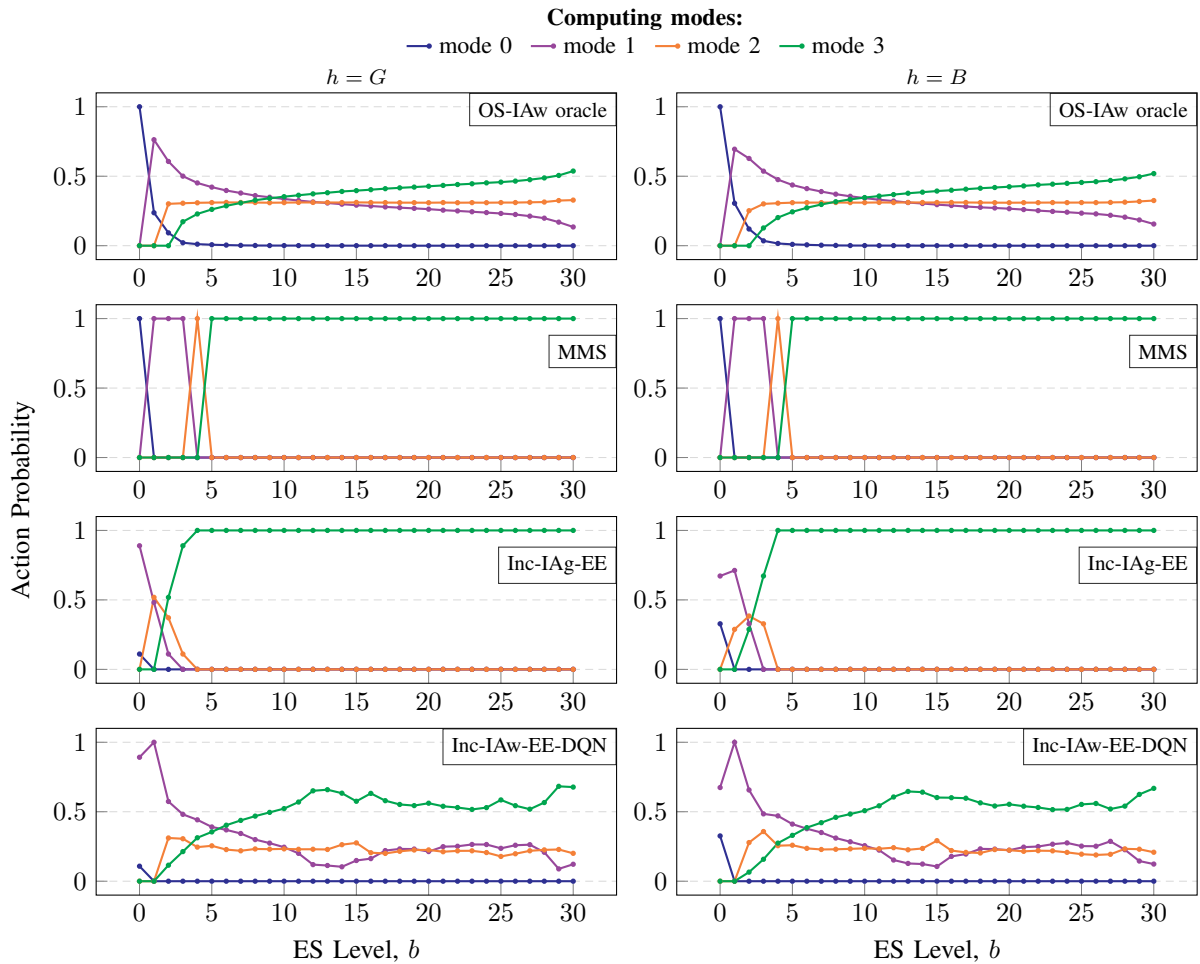}
    \caption{Numerical policies expressed in terms of probability over the actions (y-axis) conditioned on the \gls*{es} level (x-axis) and \gls*{eh} state (columns). Numerical values are computed for each controller with the appropriate algorithm using the following parameter setting: for $b_{\text{max}}=30$, $p^B_{\rh}=0.5$, $p^G_{\rh}=0.9$, $p^G_{\re^H}=0.8$, $p^B_{\re^H}=0$, $\gamma=0.9$.}
    \label{fig:policies}
\end{figure*}

\subsection{Policy Comparison}\label{sec:policy_comparison}
We compare the numerical policies obtained for \gls*{os}-\gls*{iaw}-oracle, \gls*{mms}, \gls*{inc}-\gls*{iag}-\gls*{ee} and \gls*{inc}-\gls*{iaw}-\gls*{ee} (\gls*{dqn}), using a discount factor of $\gamma=0.9$, $b_{\text{max}}=30$, $p^B_{\rh}=0.5$, $p^G_{\rh}=0.9$, $p^G_{\re^H}=0.8$, and $p^B_{\re^H}=0$. As discussed in Section~\ref{sec:result_ee_design}, we adopt a linearly increasing cost with the computation complexity. Precisely, the energy cost of the $k$-th computing mode is $u(k)=k$, $k=0, 1, 2, 3$, where $k=0$ refers to the initial random predictor.

All the policies we consider are deterministic, 
but they differ in their operational granularity. In fact, within the interval $[t_n, t_{n+1})$, the \gls*{os} controller takes a single action to decide a computing mode while the \gls*{inc} controller takes $T-1$ sub-actions to sequentially decide the optimal computing mode. Therefore, to compare the controllers' policies, we compute the fraction of input samples being processed at the $k$-th computing mode $\eta_k(b,h)$ when $\rb_{t_n}=b$ and $\rh_{t_n}=h$ at the beginning of each slot $[t_n, t_{n+1})$, $\forall n$. This can be interpreted as the empirical probability of selecting computing mode $k$ in $(b,h)$.\\
\indent In a given state $(b,h)$ the \gls*{mms} selection does not have any additional randomness, hence $\eta_k(b,h)=1$ iff $k=\pi^*_{\text{\gls*{mms}}}(b,h)$. Conversely, the \gls*{os}-\gls*{iaw} (oracle) inherits the randomness from the input samples. Thus, we compute $\eta_k(b,h)$ as the fraction of samples in $ \mathcal{D}_{\text{test}}$ being processed at the $k$-th mode when the controller follows $\pi^*$ in $(b,h)$, that is
\begin{align}
    \eta_k(b,h) = \frac{1}{\lvert \mathcal{D}_{\text{test}} \rvert}\sum_{i=1}^{\lvert  \mathcal{D}_{\text{test}} \rvert} \mathds{1}_{\{\vz_i\in\mathcal{Z}_k(b,h)\}}.
\end{align}
For the \gls*{inc}-\gls*{iag}-\gls*{ee} controller, compared to the \gls*{mms} one, the randomness stems from the intermediate energy arrivals which determine the evolution of $\rb_t$ and $\rh_t$. So, we have
\begin{align*}
    \eta_k(b,h) = \mathbb{P}\Big( \sum_{\tau=0}^{T-1}\pi^*(\rb_{t_n+\tau},\rh_{t_n+\tau})=k\mid \rb_{t_n}=b, \rh_{t_n}=h \Big).
\end{align*}
A closed-form derivation is provided in Appendix~\ref{ap:exit_prob_inc}. Finally, the \gls*{inc}-\gls*{iaw}-\gls*{ee} (\gls*{dqn}) controller is instance-aware and incremental; hence, it is affected by both the randomness of input instances and intermediate energy arrivals. To compute $\eta_k(b,h)$, we sample a sub-action trajectory $\Psi(b,h,z)=\{\alpha_\tau=\pi^*(\rb_{\tau},\rh_{\tau}, \rz^{(\tau) })\}_{\tau=0}^{T-1}$ such that $\rb_0=b$, $\rh_0=h$ and $\rz_0=z$. 
Hence,
\begin{equation}
    \eta_k(b,h) = \frac{1}{\lvert  \mathcal{D}_{\text{test}} \rvert}\sum_{i=1}^{\lvert  \mathcal{D}_{\text{test}} \rvert} \mathds{1}_{\big\{\sum_{\alpha\in\Psi(b,h,z_i)}\alpha\,=\,k\big\}}.
\end{equation}

Figure~\ref{fig:policies} shows the comparison in terms of computing mode (exit or model selection) probabilities for the policies of the four controllers mentioned above. %
The general trend indicates that the likelihood of selecting earlier exits (less energy-demanding models) increases as the \gls*{es} level decreases. Conversely, deeper exits (more energy-demanding models) are preferred when the \gls*{es} level is high. Moreover, we observe slightly less conservative policies when the computing mode is selected in (G)ood rather than (B)ad harvesting conditions. Such discrepancy is more noticeable for incremental controllers, where the action selection is optimized over a finer time granularity, thus more sensitive to modest energy variations.

\subsubsection{Comparing One-shot and Incremental Controllers}
Incremental approaches (\gls*{inc}-\gls*{iag}-\gls*{ee}, \gls*{inc}-\gls*{iaw}-\gls*{ee}-DQN) are generally less conservative at low \gls*{es} levels compared to one-shot methods (\gls*{os}-\gls*{iaw}-oracle, \gls*{mms}).
For example, when $b=0$, the \gls*{mms} controller  is restricted to selecting mode 0 (cost-free random predictor) as the \gls*{es} is depleted. Instead, under the same conditions, the \gls*{inc}-\gls*{iag}-\gls*{ee} controller adopts a mixed strategy combining modes 0 and 1, despite $u(1)\geq 0$. Analogously, when $b=1$, the \gls*{inc}-\gls*{iag}-\gls*{ee} controller operates by combining modes 1 and 2, in contrast to the \gls*{mms} which conservatively opts for mode 1. In a similar manner, the \gls*{inc}-\gls*{iaw}-\gls*{ee}-\gls*{dqn} policy enables the selection of mode 1 even when at the beginning of the decision epoch the \gls*{es} is depleted. This flexibility is absent in the \gls*{os}-\gls*{iaw}-oracle, where the cost-free random predictor remains the only feasible option.

Generally, one-shot controllers lack foresight of future energy realizations and thus optimizes on a time-averaged fashion within the slot. Conversely, incremental controllers operate at a finer granularity, and have more detailed information about intermediate \gls*{eh} conditions to take more informed decisions. In fact, an \gls*{inc} controller waits for optimal \gls*{eh} conditions within the decision epoch duration and opportunistically selects later exits, potentially achieving better performance.
\subsubsection{Comparing Instance Agnostic and Instance Aware Controllers}
Instance-aware (\gls*{os}-\gls*{iaw}-oracle, \gls*{inc}-\gls*{iaw}-\gls*{ee}-\gls*{dqn}) controllers exhibit a more gradual transition between modes than their instance-agnostic counterparts (\gls*{mms}, \gls*{inc}-\gls*{iag}-\gls*{ee}). For instance, although, \gls*{os}-\gls*{iaw}-oracle is model-based and \gls*{inc}-\gls*{iaw}-\gls*{ee}-\gls*{dqn} is model-free with more noisy mode probabilities, they both show a similar behaviour, starting with high probability for modes 0 and 1, and gradually shifting towards modes 2 and 3 as $b$ increases. On the other hand, the \gls*{mms} drastically switches between modes. Interestingly, at higher \gls*{es} levels ($b\geq 5$) the most expensive mode (mode 3) becomes dominant for both instance-agnostic methods. This is because \gls*{iag} controllers, lacking per-instance confidence, relies on average confidence, pushing for costly modes even when a less energy-intensive mode could yield good performance. This explains the \gls*{iaw} controller's strategy of processing only a fraction of input samples using the most energy-intensive mode when the \gls*{es} is high, in favour of less energy-hungry modes with sufficient accuracy.

\begin{figure*}[ht]
    \centering
    \input{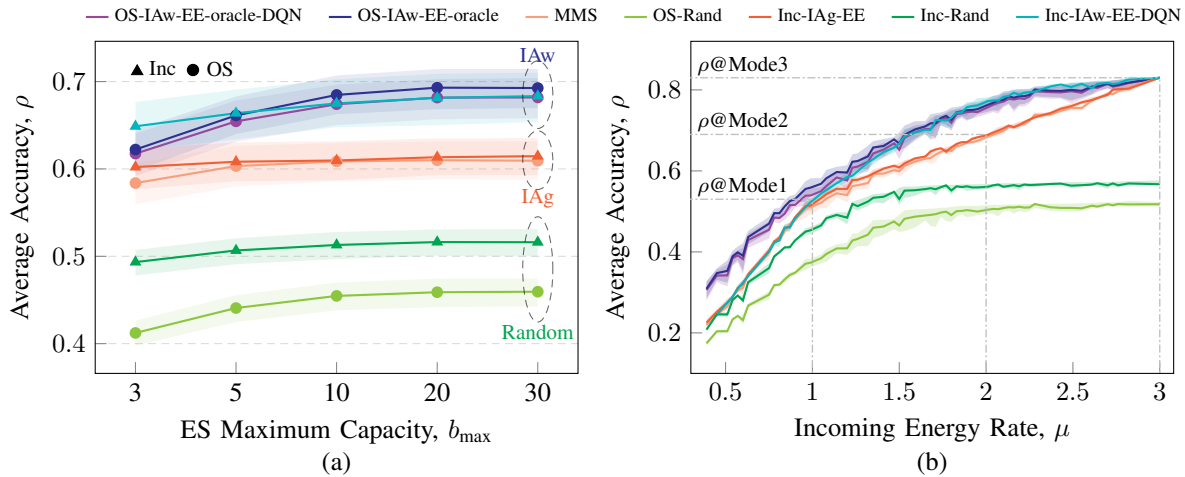}
    \caption{(a) Average accuracy $\rho_\pi$ computed as in \eqref{eq:accuracy} for each controller policy $\pi$ as a function of the \gls*{es} capacity $b_{\text{max}}$. Triangle markers identify \glsdesc*{inc} methods, while circle markers the \glsdesc*{os} ones. (b) Average accuracy $\rho_\pi$ as a function of the incoming energy rate $\mu$ computed as in \eqref{eq:e_rate}.}
    \label{fig:controller_comparison}
\end{figure*}

\subsection{Performance Comparison}\label{sec:controller_comparison}

This section provides a comprehensive comparison of various policies for \gls*{ee} and \gls*{mms} under different energy conditions. Precisely, we compare the performance in terms of accuracy $\rho_\pi$ accrued by the controller policy $\pi$ continuously performing a classification task, that is 
\begin{align}\label{eq:accuracy}
    \rho_\pi \triangleq \frac{1}{T}\sum_{t_n=1}^T \mathds{1}_{\{\hat{\ry}^{(\ra_{t_n})}_{t_n}=\ry_{t_n}\}}\ \Big\lvert\ \ra_{t_n}\sim \pi.
\end{align}
Each policy is obtained by solving a cumulative discounted reward problem with $\gamma=0.9$. The accuracy is computed for $30$ episodes of length $T=5000$, using a combination of the following hyperparameters: $p_\rh^G\in\{0.5,0.7,0.9\}$, $p_\rh^B\in\{0.3, 0.5, 0.9\}$, $p^G_{\re^H}\in\{0.3, 0.7, 0.8, 1\}$ $p^B_{\re^H}\in\{0, 0.2, 0.3, 0.5\}$, and $b_{\text{max}}=\{3,5,10,20,30\}$. For the analysis, we focus on the accuracy performance as function of the \gls*{es} maximum capacity and the incoming energy rate.

\subsubsection{Average Accuracy versus $b_{\text{max}}$}
Figure~\ref{fig:controller_comparison}(a) illustrates the variation in the average accuracy achieved by the controller using a policy $\pi$ in those environments (one for each combination of  $p_\rh^G, p_\rh^B, p^G_{\re^H}, p^B_{\re^H}$) that have the same \gls*{es} capacity $b_{\text{max}}$. Generally, higher \gls*{es} capacities with respect to the highest cost of the available modes lead to improved accuracy across all policies. A strategy based on random action selection underperforms compared to all the other methods that employ informed decision-making. Moreover, across all the \gls*{es} capacities, \gls*{iaw} controllers consistently demonstrate higher accuracy than \gls*{iag} ones, with performance improvements of up to 8\% large $b_{\text{max}}$ values. Notably, incremental approaches outperform one-shot counterparts at \gls*{es} capacities comparable to the full-model energy requirements. In fact, the incremental controller can strategically delay its decision within the designated epoch, awaiting favorable \gls*{eh} conditions. This opportunistic strategy enables the selection of subsequent exits, which lead to enhanced performance. Indeed, the \gls*{inc}-\gls*{iaw}-\gls*{ee}-\gls*{dqn}, which combines incremental decision-making and instance awareness, consistently dominates all the other policies in terms of accuracy when $b_{\text{max}}\in\{3,5\}$. For higher values $b_{\text{max}}\in\{10,20,30\}$, it even approaches the accuracy of \gls*{os}-\gls*{iaw}-oracle, while matching that of \gls*{os}-\gls*{iaw}-\gls*{dqn}. As a consequence, the small performance gap between \gls*{os}-\gls*{iaw}-oracle and model-free controllers seems to represent the performance cost a system has to pay when the environment transitions are unknown (sub-optimality gap).


\subsubsection{Average Accuracy versus Incoming Energy Rate}
The environment model used in this paper is described by a significant number of parameters, i.e., $\{p_\rh^G, p_\rh^B, p^G_{\re^H}, p^B_{\re^H}, b_{\text{max}}\}$. Each environment model is associated with an incoming energy rate $\mu$, computed as in \eqref{eq:e_rate}.
Let $U_\mu$ be the set of environments with the same $\mu$. 
The analysis in \cite{bullo} shows that $\mu$ is not sufficient to summarize the whole behaviour of an environment: in general, for $u_1\neq u_2$, $u_1,u_2\in U_\mu$, the respective optimal policies $\pi^*({u_1})$ and $\pi^*({u_2})$ differ. However, we notice that the difference in performance between $\pi^*({u_1})$ and $\pi^*({u_2})$ is limited $\forall u\in U_\mu$. Therefore, we study the impact of the incoming energy rate $\mu$ on the controller performance $\rho$ by computing the average accuracy over $U_\mu$ for the different rates $\mu$, obtained from each combination of the parameters listed at the beginning of Section~\ref{sec:controller_comparison}.\\
\indent Figure~\ref{fig:controller_comparison}(b) displays the accuracy of different policies as a function of $\mu$. We highlight in the plot the test accuracies of policies choosing always a fixed mode $i$, $\rho$@Mode$i$, and their minimum operating $\mu$, corresponding to $\{(1,0.53),(2,0.69),(3,0.83)\}$.
As the incoming energy rate increases, as expected, the accuracy achieved by all the controllers improves. Random strategy underperforms compared to the informed methods.  Moreover, across all the incoming energy rates, \gls*{iaw} policies consistently outperform \gls*{iag} counterparts, with performance improvements reaching up to approximately 5\%. Incremental approaches and their one-shot counterparts have comparable performance as a function of $\mu$. To explain this, first note that under the same rate $\mu$, the average amount of energy available over time is the same for both controllers. Moreover, for every input instance, \gls*{inc} schemes optimize sub-action selection within a fixed horizon of $T$ slots. Hence, waiting for better \gls*{eh} conditions comes at the cost of reducing the remaining time for computation. Therefore, the \gls*{inc} schemes' advantage of exploiting frequent observations to adjust sub-actions in response to short-term state variations does not necessarily translate into significantly higher accuracy. This is especially the case when performance is measured by long-term average criteria, that can attenuate short-term stochastic variations.\\
\indent The accuracy of \gls*{iag} controllers exhibit a piece-wise linear structure as a function of $\mu$, with breakpoints at $\mu=1,2$. Due to the unavailability of the per-instance confidence, these controllers always select the $i$-th computing mode when $\mu=u(i)$, and tend to linearly combine the frequency of using mode $i$ and $j$, $i<j$, when $u(i)<\mu<u(j)$. Within the \gls*{iag} controllers, \gls*{mms} and \gls*{inc}-\gls*{iag}-\gls*{ee} behave similarly.\\
\indent Instance-aware methods exploit the knowledge of per-instance confidence to push the performance beyond that of \gls*{iag} controllers. The accuracy $\rho$ as a function of $\mu$ is superlinear, with the maximum improvements of approximately $5\%$ over \gls*{iag} controllers reached around $\mu=2$. At this rate, \gls*{iaw} methods achieve an accuracy of approximately $0.75$ while \gls*{iag} methods reach approximately $0.7$. In general, for the same accuracy target, \gls*{iaw} controllers need lower incoming energy rates than their \gls*{iag} counterparts. For example, the minimum energy rate required by \gls*{iaw} controllers to achieve an accuracy target equal to $\rho$@Mode2 is approximately $1.5$, while for \gls*{iag} ones it is $2$. \gls*{iaw} methods dominate until the average incoming energy reaches $3$ units per slot, which is enough to power the entire \gls*{dnn}. In this unconstrained scenario, the optimal behavior for all controllers is to consistently choose exit 3, which has an energy cost of $3$ and yields the highest accuracy.
Overall, the \gls*{os}-\gls*{iaw}-oracle controller (and similarly the \gls*{os}-\gls*{iaw}-\gls*{dqn}-oracle) leads in performance, showing high accuracy even for incoming energy rates lower than the cost of the first computing mode ($u(1)=1$). Particularly, for \(0 \leq \mu < 1\), where the average available energy per slot is only sufficient to power a random-guesser, the instance-aware oracle controllers can leverage confidence information to achieve higher accuracy per unit of incoming energy. In the same energy rate range, the \gls*{os}-\gls*{iaw}-\gls*{dqn} behaves as if confidence-awareness is not available, showing the same behaviour as \gls*{iag} controllers.

\subsection{Key Observations}
Each policy can be interpreted as establishing dynamic thresholds for each computing mode, which vary with the \gls*{es} level. For instance-aware methods, this dynamic nature is also influenced by the distribution of input confidences.
Although the likelihoods produced by a \gls*{dnn}, even when calibrated, cannot be strictly considered as true probabilities, their integration into the decision-making process leads to more informed decisions, achieving higher accuracy with reduced energy consumption. This approach favors the selection of more energy-intensive modes only when necessary, such as when an input instance presents high complexity for the current task and DNN model. Furthermore, when a system with constrained $b_{\text{max}}$ operates in a low-incoming energy regime, incremental controllers outperform their one-shot counterparts. In our experiments, this is observed when the cost of the most expensive computing mode exceeds 60\% of the system battery capacity \(b_{\text{max}}\). Instead, in high-incoming energy regimes, and for an unconstrained $b_{\text{max}}$, the performance of one-shot and incremental controllers are comparable.
In conclusion, if the distribution of input samples in 
the available dataset is statistically representative of the one in the user application, incorporating \gls*{dnn} output likelihoods into the decision-making process is recommended to enable instance-awareness. Additionally, our findings suggest that adopting multi-exit \glspl*{dnn} as a proxy for inference adaptation is advantageous, as it enables not only instance-awareness, but also incremental decision-making strategies with better performance with relatively small ES capacity.

\section{Conclusion and Future Directions}\label{sec:conclusion}
In this paper, we addressed the challenges associated with adapting neural inference workloads to the available energy envelope and \gls*{es} constraints in \glspl*{ehd}. We developed a comprehensive framework for optimal control in dynamic \glspl*{dnn} by studying \gls*{mms} and \gls*{ee} strategies. These approaches were tailored to leverage instance-aware and instance-agnostic control schemes, operating at both one-shot and incremental granularities. Our main contributions include establishing the optimal policy structure for the MmS system, demonstrating its monotonicity in \gls*{es} levels for energy and memory efficiency. We formulated a one-shot oracle controller using per-instance exit confidences with theoretical guarantees and designed an approximate \gls*{vi} algorithm to estimate optimal policies using empirical confidence distributions. Additionally, we developed a sub-optimal policy based on a lightweight \gls*{dqn} for incremental instance-aware \gls*{ee} selection, named \gls*{inc}-\gls*{iaw}-\gls*{ee}. In conclusion, we conducted comprehensive empirical evaluations comparing our control schemes on a custom multi-exit EfficientNet model tested on the TinyImageNet dataset, analyzing accuracy under various \gls*{es} capacities and energy rates. For the extensive range of parameter values examined in this work, simulations indicates that \gls*{iaw} and \gls*{inc} control schemes can significantly enhance accuracy, encouraging the adoption for practical designs. In fact, \gls*{inc} approaches proved to be more efficient, achieving higher accuracy in scenarios with limited \gls*{es} capacity, and \gls*{iaw} schemes consistently outperformed their \gls*{iag} counterparts.

Future research can expand on several interesting directions. In many applications, collected data exhibits temporal correlations, which 
can be leveraged to decide the computing mode.
While this paper assumes a sensing apparatus operating at a fixed sampling rate, a more general scenario involves input instances being collected at variable rates. In such cases, the intelligent management of time resources becomes essential. In fact, \gls*{ee} strategies can be employed to halt the processing earlier, thereby conserving both energy and time. This approach presents a significant trade-off between energy consumption and the time required for computation. However, effectively managing this trade-off necessitates a theoretical model for input arrivals, (e.g., a Poisson process).

{\appendices
\section{Proof of Theorem~\ref{th:pwl_vf}}\label{ap:proof_pwl_vf}
\begin{proof}[Proof of Theorem~\ref{th:pwl_vf}]
	By definition of Bellman-optimal value function we have
	\begin{align}
                v^*(\vs,\vz) &= \max_{a\in\mathcal{A}_\vs} \Big\{ r(\vs,\vz,a) +
                \gamma \sum_{\vs'\in\mathcal{S}} \mathbb{P}(\vs'|\vs,a)\int_\mathcal{Z} v^*(\vs',z')\mathbb{P}(\vz')\,d\vz'\Big\}\notag\\
                &=  \max_{a\in\mathcal{A}_{\vs}}\Big\{ e^\top_a\vz + \gamma \sum_{\vs'\in\mathcal{S}} \mathbb{P}(\vs'|\vs,a)\bar{v}^*(\vs') \Big\}.
        \end{align}
        By rewriting the above expression in a vector form, i.e., $\mathbb{P}_a(\vs)\triangleq[\mathbb{P}(\vs_1|\vs,a), \dots, \mathbb{P}(\vs_{|\mathcal{S}|}|\vs,a)]^\top$, and $\bar{\vv}^*=[\bar{v}^*(\vs_1), \dots, \bar{v}^*(\vs_{|\mathcal{S}|})]^\top$, we obtain
        \begin{align}\label{eq:conditions}
               v^*(\vs,\vz) &=  \max_{a\in\mathcal{A}_{\vs}}\Big\{ \vz^{(a)} + \gamma\mathbb{P}^\top_a(\vs)\bar{v}^* \Big\}=\max_{a\in\mathcal{A}_{\vs}}q^*(\vs,\vz,a)
            = \sum_{j=0}^{K-1} \Big(e^\top_j\vz + \gamma \mathbb{P}_{a_j}(\vs)\bar{\vv}^*\Big)\mathds{1}_{\{a_j\succeq a_i, \forall i\neq j\}},
	\end{align}
 where $a_j\succeq a_i$\footnote{Ties are broken arbitrarily.} means that action $a_j$ is at least as good as $a_i$ in the following sense
 \begin{align}\label{eq:preference}
     a_j\succeq a_i\quad  \Longleftrightarrow\quad  q^*(\vs,\vz,a_j)\geq q^*(\vs,\vz,a_i),\quad i\neq j.
 \end{align}
 By plugging the definition of $q^*$ into \eqref{eq:preference} and re-arranging, we obtain $\forall\, i=0,\dots,K-1$
 \begin{align}\label{eq:z_sys_inequalities}
     z^{(a_j)} - z^{(a_i)} \geq\gamma\Big(\mathbb{P}_{a_i}(\vs) - \mathbb{P}_{a_j}(\vs)\Big)\bar{v}^*,\ i\neq j.
 \end{align}
 Let $\delta^{*}_{ij}(\vs)\triangleq \gamma\big(\mathbb{P}_{a_i}(\vs) - \mathbb{P}_{a_j}(\vs)\big)\bar{v}^*$. Note that
\begin{align}\label{eq:opposite_deltas}
    \delta_{ij}(\vs)&=-\delta_{ji}(\vs), \\
    \delta_{ij}(\vs)&=\delta_{ik}(\vs)+\delta_{kj}(\vs),\quad  \forall i,j=1,\dots,K.
\end{align}
Without loss of generality we choose a reference index, say $k=0$. In this way, every $\delta_{ij}(\vs)$ can be rewritten as a function of $\delta_{0i}, i=0,\dots,K-1$, as follows:
\begin{align}
     \delta_{ij}(\vs) &= \delta_{i0}(\vs)+\delta_{0j}(\vs)\\
     \delta_{ij}(\vs) &= -\delta_{0i}(\vs)+\delta_{0j}(\vs).\label{eq:deltas_props}
\end{align}
By definition \eqref{eq:z_sys_inequalities}, $\delta_{ii}=0,\forall i$, therefore, $K-1$ values are sufficient to represent every $\delta_{ij}(s)$.

According to \eqref{eq:conditions}, for each state $\vs\in\mathcal{S}$, $\mathcal{Z}$ is partitioned into $K$ subsets 
\begin{align*}
    \mathcal{Z}_j(\vs) = \{\vz\in\mathcal{Z}: \mM_j \vz\geq\mF_j\vdelta(\vs)\},\quad j=0,\dots,K-1,
\end{align*}
where $\vdelta(\vs)=[\delta_{0i}, i=1,\dots,K-1]$, $\mF_j\in\{0,\pm1\}^{K-1\times K-1}$ maps $\vdelta(\vs)$ into the corresponding thresholds according to \eqref{eq:deltas_props}, and $\mM_j\in\{0,\pm 1\}^{K-1\times K}$ is a matrix formed by a negative identity matrix $-\mI_{K-1}$ and a column vector of ones inserted at column index $j$. Therefore,
\begin{align}
    v^*(\vs,\vz) = e^\top_j\vz + \gamma \mathbb{P}_{a_j}(\vs)\bar{\vv}^*,\vz\in\mathcal{Z}_j(\vs),
\end{align}
which is linear in $\vz\in\mathcal{Z}_j(\vs), \forall j=0,\dots,K-1$.
\end{proof}

\section{Proof of Theorem \ref{th:mms_optimal_policy}}\label{ap:proof_opt_monotone}
\begin{proof}[Proof of Theorem \ref{th:mms_optimal_policy}]
We begin by characterizing the transition probabilities, where $\mathbb{P}\big(b',h' \lvert b,h,a\big)$ indicates 
\begin{align}
    \mathbb{P}\big(\rb_{t_{n+1}}=b',\rh_{t_{n}}=h'\big\lvert\rb_{t_n}=b,\rh_{t_{n-1}}=h,\ra_{t_n}=a \big),
\end{align} 
and $p_\rh(h'|h)\triangleq\mathbb{P}(\rh_{{t_n}}=h'|\rh_{t_{n-1}}=h)$, $\forall n\in\mathbb{N}$. Then, let $p_{\rb'}(b'|b,h,a)\triangleq \mathbb{P}(\rb_{t_{n+1}}=b'|\rb_{t_n}=b,\rh_{{t_n}}=h,\ra_{t_n}=a)$, where $\rb_{t_{n+1}}=\min([\rb_{t_n}-u(\ra_{t_n})]^++\re_{t_n}^H), b_{\text{max}})$, $\forall n\in\mathbb{N}$. Hence,
\begin{align*}
    \begin{split}
        \mathbb{P}\big(b',h' \lvert b,h,a\big) &= p_{\rb'}(b'|b,h',a)\,p_\rh(h'|h)\\
        & = \sum_{e\in \mathcal{E}^H} \mathds{1}_{\{b'=\beta(b,a,e)\}}\,p_\rh(h'|h)\,p^\rh_{\re^H}(e|h'),
    \end{split}
\end{align*}
where $\beta(b,a,e)\triangleq \min([b-u(a)]^++e), b_{\text{max}})$.

For brevity, let $\mathbb{P}(e,h'|h)\triangleq p_\rh(h'|h)\,p^\rh_{\re^H}(e|h')$, $v_h^*(b)\triangleq v^*_{\text{\gls*{mms}}}(h,b)$, $\beta^e_{ij}\triangleq\beta(b_i,a_j,e)$ and $r_{ij}\triangleq r(b_i,a_j)$. In order for \eqref{eq:opt_mms_policy} to hold, it is sufficient to prove that
    \begin{align}
        \begin{split}
            q_h^*(b,a) &= r(b,a) + \gamma\sum_{h'\in\mathcal{H}}\sum_{b'=0}^{b_{\text{max}}} \mathbb{P}\big(b',h' \lvert b,h,a\big) v_{h'}^*(b')\\
            &= r(b,a) + \gamma \sum_{e,h'} \mathbb{P}(e,h'|h)\,v_{h'}^*\big(\beta(b,a,e)\big)
        \end{split}
    \end{align}
    has non-decreasing differences (superadditive) \cite{puterman}, that is for $b_2\geq b_1$ and $a_2\geq a_1$, $a_2, a_1\in\mathcal{A}_{b_1}$,
    \begin{align}\label{eq:superadditivity}
        q^*_h(b_2,a_2) - q^*_h(b_2,a_1) \geq q^*_h(b_1,a_2) - q^*_h(b_1,a_1).
    \end{align}
    Considering that the optimal value function $v^*_{h}$ is the unique fixed point of $(\mathcal{T}v)(b) = \max_{a\in\mathcal{A}_b}q_h(b,a)$, i.e., $\lim_{\ell\to\infty}(\mathcal{T}^\ell v^0)(b) = v^*_{h}(b)$ for any arbitrary initial function $v^0$, we show by means of an inductive argument over $\ell=1,2,\dots$ that \eqref{eq:superadditivity} holds.

    First, we note that for any \gls*{es} level $b\in\mathcal{B}$ and feasible actions $a_1,a_2\in\mathcal{A}_b$, the reward function has constant differences, that is
    \begin{align}\label{eq:reward_constant}
        r(b,a_1) - r(b,a_2) = r(b^+,a_1) - r(b^+,a_2),\ \forall b^+\geq b.
    \end{align}
    Hence, for $v^0(b)=0,\ \forall b$, \eqref{eq:superadditivity} is satisfied since $q^{0}_h(b,a_2) - q^{0}_h(b,a_1)=0$, $\forall b$. Assume that \eqref{eq:superadditivity} holds for $q^{\ell}_h$. Because of \eqref{eq:reward_constant}, this is equivalent to 
    \begin{align*}
        \begin{split}
            \mathbb{E}_{\re,\rh'}\Big[v_{\rh'}^\ell\big(\beta_{22}^\re\big)-
            v_{\rh'}^\ell\big(\beta_{21}^\re\big)\Big]\geq  \mathbb{E}_{\re,\rh'}\Big[v_{\rh'}^\ell\big(\beta_{12}^\re\big)-
        v_{\rh'}^\ell\big(\beta_{11}^\re\big)\Big].
        \end{split}
    \end{align*}
    Then
    \begin{align}
        \begin{split}
            \delta^{\ell+1}_q(b_2) &= q^{\ell+1}_h(b_2,a_2) - q^{\ell+1}_h(b_2,a_1)\\
            &= (r_{22}-r_{21}) + \gamma\mathbb{E}_{\re,\rh'}\Big[v_{\rh'}^\ell\big(\beta_{22}^\re\big)-v_{\rh'}^{\ell}\big(\beta_{21}^{\re}\big)\Big]\\
            &\geq (r_{12}-r_{11}) +\gamma\mathbb{E}_{\re,\rh'}\Big[v_{\rh'}^\ell\big(\beta_{12}^{\re}\big)-v_{\rh'}^\ell\big(\beta_{11}^{\re}\big)\Big]\\
            &= q^{\ell+1}_h(b_1,a_2) - q^{\ell+1}_h(b_1,a_1) =  \delta^{\ell+1}_q(b_1)
        \end{split}
    \end{align} 
    Therefore, since superadditivity of $q_h$ is preserved trough a Bellman update, $q^*_h$ must be superadditive.
\end{proof}

\section{Exit Probability for \gls*{inc}-\gls*{iag}-\gls*{ee} Controllers}}\label{ap:exit_prob_inc}
Let us define the set of discrete states indicating that the system is at $\mathfrak{t}$-th processing stage as 
$$
    \mathcal{S}_{\mathfrak{t}} \triangleq \{s=(\rb, \rh, \xi, \tau)\in\mathcal{S}:\tau=\mathfrak{t}\}.
$$
We refer to $\mathcal{S}_{0}$ as the set of initial states. We allow $\tau=K-1$ and we define a fictitious set of final states, $\mathcal{S}_{K-1}$, where the reward corresponding to the previously selected exit is accrued, and no actions are selected. Note that if $\tau=0$ then $\xi=0$, since the processing of an input instance always starts from the $0$-th exit, therefore the number of initial states is $N_0 = |\mathcal{B}||\mathcal{H}|$ and the number of final states is $N_{K-1}=|\mathcal{B}||\mathcal{H}|K$.
For every $\vs_0\in\mathcal{S}_{0}$, we want to compute the probability of selecting the $k$-th exit in the remaining $K-1$ slots, prior to the next sample arrival.
Let $\pi_{\text{inc}}^*$ the $\varepsilon$-optimal policy found by any dynamic programming algorithm, and $\mathbb{P}_{\pi_{\text{inc}}^*}$ the transition probability matrix induced by such policy.
Without any loss of generality, we assume that the first $N_0$ states in $\mathbb{P}_{\pi_{\text{inc}}^*}$ are the
initial states, indexed by $i=0,\dots,N_0-1$. Similarly, the states indexed by $i=N_0,\dots,N_0+N_{K-1}-1$ are final states. Hence, the $\mathfrak{t}$-step transition probability matrix $\mathbb{P}^{(\mathfrak{t})}_{\pi_{\text{inc}}^*}$, where for the $n$-th input instance, the time instant $\mathfrak{t}$ corresponds to $t=t_n+\mathfrak{t}$, can be partitioned as follows
$$
        \mathbb{P}^{(\mathfrak{t})}_{\pi_{\text{inc}}^*} = \begin{bmatrix}
                    \mathbb{P}^{(\mathfrak{t})}_{0,0} & \mathbb{P}^{(\mathfrak{t})}_{0,{K-1}} & \mathbb{P}^{(\mathfrak{t})}_{0,\tilde{s}}\\
                    \mathbb{P}^{(\mathfrak{t})}_{{K-1},0} & \mathbb{P}^{(\mathfrak{t})}_{{K-1},{K-1}} & \mathbb{P}^{(\mathfrak{t})}_{{K-1},\tilde{s}}\\
                    \mathbb{P}^{(\mathfrak{t})}_{\tilde{s},0} & \mathbb{P}^{(\mathfrak{t})}_{\tilde{s},{K-1}} & \mathbb{P}^{(\mathfrak{t})}_{\tilde{s},\tilde{s}}
                \end{bmatrix},
$$
where $\tilde{\vs}\notin \mathcal{S}_{0}$, $\tilde{s}\notin\mathcal{S}_{K-1}$, and, with a slight abuse of notation, we write $\mathbb{P}^{(\mathfrak{t})}_{i,j}\triangleq \mathbb{P}^{(\mathfrak{t})}(\vs'\in\mathcal{S}_{j}|\vs\in\mathcal{S}_{i})$. In particular, we are interested in $\mathbb{P}^{(K-1)}_{0,{K-1}}$. Note that for every $0<\mathfrak{t}\leq K-1$, $\mathbb{P}^{(\mathfrak{t})}_{0,0}=0$, since for every $s\in\mathcal{S}_0$, $\tau_{t_n}=0$ and, after $\mathfrak{t}$ steps we have $\tau_{t_n+\mathfrak{t}}=\tau_{t_n}+\mathfrak{t}\neq 0$; if $\mathfrak{t}=K-1$, $\mathbb{P}^{(K-1)}_{0,\tilde{s}}=0$, since $\tau_{t_n+K-1}= \tau_{t_n}+K-1\neq 0$.
Therefore, $\mathbb{P}^{(K-1)}_{0,{K-1}}$ is stochastic, since $\mathbb{P}^{(K-1)}_{\pi_{\text{inc}}^*}$ does.
Finally, let the set of terminal states that end with exit $k$ be 
$$
    \mathcal{S}_{k,K-1}\triangleq\{\vs=(\rb, \rh, \xi, \tau)\in\mathcal{S}:\xi=k,\tau=K-1\}.
$$
The probability of choosing exit $k$ whenever $s\in\mathcal{S}_{0}$, is
$$
    \eta_k(\vs) \triangleq \sum_{\vs'\in\mathcal{S}_{k,K-1}}\mathbb{P}^{(K-1)}_{\pi_{\text{inc}}^*}(\vs,\vs')=\mathbf{1}^\top\mathbb{P}^{(K-1)}_{0,{K-1}}(\vs),
$$
where $\mathbb{P}^{(K-1)}_{0,{K-1}}(\vs)$ is the row of $\mathbb{P}^{(K-1)}_{\pi_{\text{inc}}^*}$ corresponding to $\vs$.


\vspace{0.5cm}
\bibliographystyle{IEEEtran}
\bibliography{ref}

\end{document}